\documentclass[a4paper,fleqn]{cas-dc}

\usepackage[authoryear,longnamesfirst]{natbib}

\usepackage{amsmath, amsthm, amssymb, hyperref, color, verbatim}
\usepackage{graphicx,tabularx}
\usepackage{bm}
\usepackage[ruled,vlined]{algorithm2e}

\usepackage[shortlabels]{enumitem}
\usepackage{graphicx}
\usepackage{makecell}
\usepackage{algpseudocode}
\usepackage{lipsum}
\usepackage[utf8]{inputenc}
\usepackage{geometry}
\usepackage{algpseudocode} 

\newtheorem{theorem}{Theorem}
\numberwithin{theorem}{section}
\newtheorem{proposition}[theorem]{Proposition}
\newtheorem{lemma}[theorem]{Lemma}
\newtheorem{corollary}[theorem]{Corollary}

\newtheorem{definition}[theorem]{Definition}

\newtheorem{remark}[theorem]{Remark}
\newtheorem{example}[theorem]{Example}

\newcommand{\RR}{\mathbb{R}}
\newcommand{\R}{\mathbb{R}}

\newcommand{\TPP}{\RR^d \!/\mathbb R {\bf 1}}

\DeclareMathOperator*{\argmax}{arg\,max}

\def\tsc#1{\csdef{#1}{\textsc{\lowercase{#1}}\xspace}}
\tsc{WGM}
\tsc{QE}

\begin{document}
\let\WriteBookmarks\relax
\def\floatpagepagefraction{1}
\def\textpagefraction{.001}

\shorttitle{Tropical Support Vector Machines}    

\shortauthors{Yoshida et al.}  

\title [mode = title]{Tropical Support Vector Machines: Evaluations and Extension to Function Spaces}  



%

\author[1]{Ruriko Yoshida}


\tnotetext[0]{RY is partially funded by NSF DMS 1916037. HM is partially funded by JSPS KAKENHI 19H04987, 19H05024 and 19K06957. KM is partially funded by JSPS KAKENHI 18K11485}

\ead{ryoshida@nps.edu}

\ead[url]{polytopes.net}


\affiliation[1]{organization={Department of Operations Research, Naval Postgraduate School},
            city={Monterey},
            postcode={93943}, 
            state={CA},
            country={USA}}
\author[2]{Misaki Takamori} 





\affiliation[2]{organization={Graduate School of Science and Technology, Kwansei Gakuin University},
            city={Sanda},
            postcode={669-1337},
            state={Hyogo},
            country={Japan}}

\author[3]{Hideyuki Matsumoto} 





\affiliation[3]{organization={Graduate School of Medicine, Osaka City University},
            postcode={558-8585}, 
            state={Osaka},
            country={Japan}}

\author[2]{Keiji Miura}[orcid=0000-0002-9258-6541]


\ead{miura@kwansei.ac.jp}


\cormark[1]
\cortext[1]{Corresponding author}



\begin{abstract}
Support Vector Machines (SVMs) are one of the most popular supervised learning models to classify using a hyperplane in an Euclidean space.
Similar to SVMs, tropical SVMs classify data points using a tropical hyperplane under the tropical metric with the max-plus algebra.
In this paper, first we show generalization error bounds of tropical SVMs over the tropical projective torus.
While the generalization error bounds attained via
{\color{black}Vapnik-Chervonenkis (}VC{\color{black})} dimensions in a distribution-free manner still depend on the dimension, we also show {\color{black}numerically and} theoretically by extreme value statistics that the tropical SVMs for classifying data points from two Gaussian distributions as well as empirical data sets of different neuron types are fairly robust against the curse of dimensionality.
Extreme value statistics also underlie the anomalous scaling behaviors of the tropical distance between random vectors with additional noise dimensions.
Finally, we define tropical SVMs over a function space with the tropical metric. 
\end{abstract}


\begin{highlights}
\item We obtained generalization error bounds of tropical Support Vector Machines (SVMs) via {\color{black}the} Vapnik-Chervonenkis dimensions {\color{black} of tropical hyperplanes using Tropical Radon Lemma by Jaggi et al. (2008).} 
\item We demonstrate{\color{black}d} theoretically by extreme value statistics that the tropical SVMs
are robust against the curse of dimensionality
for classifying data points from two Gaussian distributions
as well as experimentally recorded activities from different neuron types.
\item We define{\color{black}d} tropical SVMs over a function space {\color{black} to enable the classification of curves, which is a task we encounter frequently in practice, such as classifications on neuronal tuning curves}.
\end{highlights}

\begin{keywords}
Extreme Value Statistics \sep
Function Spaces \sep
Max-plus Algebra \sep
Supervised Learning \sep
Tropical Geometry
\end{keywords}

\maketitle

\section{Introduction}

In data science, one of the well-known challenges we face is to classify data points with a large number of predictors.  For example, in image processing, in order to discriminate one object, such as missiles or face recognition, from others in images, these images are described as a vector in pixels which are typically in a very high dimensional vector space (\cite{857795}).  In bioinformatics, researchers try to classify particular diseases using high dimensional data sets such as micorarrays or SNPs (\cite{Fan3}).
\cite{Fan3} summarize challenges and difficulties with high dimensionality in classification.

Support Vector Machines (SVMs) are one of well-known supervised learning models to classify data points using a hyperplane and  introduced by
\cite{Boser92atraining} and 
\cite{Cortes95support-vectornetworks}. 
The classical SVMs introduced by
\cite{Boser92atraining} can be written as the $L_2$ norm SVMs, that is a hinge loss plus the $L_2$ norm penalty formulation. 
\cite{10.1016/j.neunet.2009.11.012} showed that including many redundant features can cause difficulties in the performance of $L_2$ norm SVM.
\cite{10.5555/645527.657467} showed that the SVM with the $L_1$ norm penalty instead of the $L_2$ norm penalty works well with variable selection as well as classification at the same time.  This is called {\em $L_1$ norm SVMs}.
{\color{black}
While there are several frameworks to analyze the generalization performance, such as covering numbers \citep{MohriRostamizadehTalwalkar18}, real log canonical thresholds \citep{Hayashi17}, mean-field regime \citep{Nitanda21}, and Langevin dynamics regime \citep{NEURIPS2020_df1a336b}, the SVMs especially allow us to derive the upper bounds for the generalization error via Vapnik-Chervonenkis (VC) dimensions \citep{Vapnik}.}
{\color{black} Although the bounds via VC dimensions may not necessarily be tight, it is still significant to have the first bound for a new variant of SVMs.}
\cite{10.5555/2946645.3053515} showed the generalization bound on the probability of errors of the $L_1$ norm SVMs and it still depends on the number of predictors, i.e., the dimension of the normal vector of the hyperplane.

Because of advances in computational algebraic geometry (\cite{ren}), tropical geometry finds applications in data science.   For example, it can be applied to principal component analysis  {\color{black}(\cite{10.1093/bioinformatics/btaa564,YZZ}) and Bayesian Networks (\cite{Tran}).  With these statistical methodologies, tropical geometry is applied to phylogenomics analyses on Apicomplexa, African coelacanth whole genome data sets, and 1089 full length sequences of hemagglutinin (HA) for
influenza A H3N2 from 1993 to 2017 in the state of New York 
obtained from the GI-SAID EpiFlu  (\cite{10.1093/bioinformatics/btaa564,YZZ}).}
\cite{Gartner} applied tropical geometry to SVMs.
Instead of using a hyperplane defined by the $L_2$ metric,
\cite{Gartner} used a {\rm tropical hyperplane} with the {\em tropical metric} with the {\em min-plus} algebra to classify the data points. 
Same as the $L_2$ norm SVMs, tropical SVMs are the {\em tropicalization} of the $L_2$ norm hard margin SVMs which maximizes the {\em margin}, the distance from the tropical hyperplane to the closest data points in terms of the min-plus algebra.
\cite{Gartner} also showed that a hard margin tropical SVM can be formulated as a linear programming problem. Then,
\cite{tang} showed necessary and sufficient conditions of the optimal solutions of the linear programming problem for finding a hard margin tropical SVM if it is feasible in terms of the max-plus algebra. They also introduced soft margin tropical SVMs 
 and showed that finding a soft margin tropical SVM can be formulated as a linear programming problem.


While there are some developments in computational sides of tropical SVMs, there has not been much in its theoretical evaluation and statistical analysis.
For example, it is unclear how the tropical metric with the max-plus algebra and tropical SVMs handle the curse of dimensionality.
Whether tropical SVMs are robust against the curse of dimensionality is an important question to answer, because we can extend tropical SVMs to various types of data including a function space.
In fact, there are many classification problems with functions as predictors, 
such as neuronal tuning curves (\cite{dayan01}), instead of vectors.

In this paper we focus on hard margin tropical SVMs with the max-plus algebra.  Like the $L_2$ norm and the $L_1$ norm SVMs, we assume that there exists the optimal tropical hyperplane defined by the normal vector $\omega \in \RR^d \!/\mathbb R {\bf 1}$ such that the probability of the loss function being positive equals to zero.    In addition, in order for the tropical metric to be well defined, we consider the \emph{tropical projective torus}, that is, the projective torus
$\mathbb R^d \!/\mathbb R {\bf 1}$, where ${\bf 1}:=(1, 1, \ldots , 1)$ is defined as the all-one vector.  Note that any vector $(v_1, \ldots, v_d) \in \mathbb R^d \!/\mathbb R {\bf 1}$ is equal to $(v_1+c, \ldots, v_d+c)$ with any scalar $c\in \mathbb R$ and $\mathbb R^d \!/\mathbb R {\bf 1}$ is isometric to $\RR^{d - 1}$.
In applications, the tropical projective torus is useful for subtracting the baseline $\RR {\bf 1}$ components from feature vectors, as we will see.

Our main contribution of this paper primarily consists of two items: 
(1) evaluation of generalization errors for tropical SVMs {\color{black}using VC dimensions for tropical hyperplanes} over $\mathbb R^d \!/\mathbb R {\bf 1}$; and 
(2) extension to tropical SVMs on a function space, which consists of the tropical metric and a set of functions from a multi-dimensional vector space to a real number.

While generalization error bounds of tropical SVMs over $\mathbb R^d \!/\mathbb R {\bf 1}$ still depends on the dimension, our simulations with a set of Gaussian distribution functions show that errors rates of tropical SVMs on $\mathbb R^d \!/\mathbb R {\bf 1}$ grow much slower than ones with  $L_2$ norm SVMs when we increase the number of predictors while we fix the sample sizes.  Then we extend a notion of tropical SVMs to a function space with the tropical metric. In fact, we show that a set of all functions from a multi-dimensional vector space to a real number with the tropical metric is a normed vector space.


This manuscript is organized as follows:
In Section \ref{sec:def}, we set up basics in tropical geometry under the max-plus algebra.
In Section \ref{sec:app1}, as an application, we show anomalous scaling behaviors of tropical distances.
In Section \ref{sec:svm}, we set up tropical SVMs with the tropical metric over the tropical projective torus as linear programming problems and show the generalization error bounds via the VC dimension of a tropical hyperplane.
In Section \ref{sec:comp}, as an application, we show theoretically by the extreme value statistics the robustness of the tropical SVMs against the curse of dimensionality in a case of two different multivariate Gaussian distributions and empirical data set of neuron types.
In Section \ref{sec:functionSpace}, we first define the tropical distance over a function space.  Then we show that the tropical distance between functions is metric and we also show that a function space with the tropical metric forms a normed vector space.  In addition, we define
tropical SVMs over a function space with the tropical metric.

\section{Definitions of Tropical Distances and Hyperplanes}\label{sec:def}

In this section, we remind readers some basics in tropical arithmetic and algebra using the max-plus algebra.   
Through this paper, we consider the tropical projective torus $\mathbb R^d \!/\mathbb R {\bf 1}$ which is isometric to $\R^{d-1}$.  For more details, see \cite{ETC,MS,tang}.

\begin{definition}[Tropical Arithmetic Operations]
In this tropical semiring $(\,\mathbb{R} \cup \{-\infty\},\boxplus,\odot)\,$, the basic tropical arithmetic operations of addition and multiplication are defined as:
$$a \boxplus b := \max\{a, b\}, ~~~~ a \odot b := a + b ~~~~\mbox{  where } a, b \in \mathbb{R}\cup\{-\infty\}.$$
Over this semiring, $-\infty$ is the identity element under addition and 0 is the identity element under multiplication.
\end{definition}

\begin{definition}[Tropical Scalar Multiplication and Vector Addition]
For any scalars $a,b \in \mathbb{R}\cup \{-\infty\}$ and for any vectors $v = (v_1, \ldots ,v_d), w= (w_1, \ldots , w_d) \in (\mathbb{R}\cup-\{\infty\})^d$, we define tropical scalar multiplication and tropical vector addition as follows:
$$a \odot v \boxplus b \odot w := (\max\{a+v_1,b+w_1\}, \ldots, \max\{a+v_d,b+w_d\}).$$
\end{definition}

{\color{black}
\begin{definition}
Suppose $S \subset \mathbb R^d \!/\mathbb R {\bf 1}$. If 
\[
a \odot v \boxplus b \odot w \in S
\]
for any $a, b \in \R$ and for any $v, w \in S$, then $S$ is called {\em tropically convex}.
\end{definition}

\begin{definition}[Tropical Convex Hull]\label{def:polytope}
Suppose we have a finite subset $V = \{v^1, \ldots , v^s\}\subset \mathbb R^d \!/\mathbb R {\bf   1}$.  The {\em tropical convex hull} or {\em tropical polytope} of $V$ is the smallest tropically-convex subset containing $V$.   It can be written as the set of all tropical linear combinations of $V$ such that:
$$ \mathrm{tconv}(V) = \{a_1 \odot v^1 \boxplus a_2 \odot v^2 \boxplus \cdots \boxplus a_s \odot v^s \mid  a_1,\ldots,a_s \in \R \}.$$
A tropical polytope of a set of two points $\{v^1, \, v^2\} \subset \mathbb R^d \!/\mathbb R {\bf   1}$ is called a {\em tropical line segment} between two points $v^1, \, v^2$.
\end{definition}

{\color{black}Note that over the tropical projective torus $\mathbb R^d \!/\mathbb R {\bf   1}$, any point $x=(x_1, \ldots , x_d) \in  \mathbb R^d \!/\mathbb R {\bf   1}$ can be also written as
\[
(x_1, \ldots , x_d) = (x_1+c, \ldots , x_d+c)
\]
for any $c \in \mathbb{R}$ by the definition of taking mod by ${\bf   1} := (1, \ldots , 1)$.  Therefore, we can assume that 
\[
(x_1, \ldots , x_d) = (x_1-x_d, \ldots , x_{d-1}-x_d, 0)
\]
for any point $x:=(x_1, \ldots , x_d) \in  \mathbb R^d \!/\mathbb R {\bf   1}$.
Therefore over this manuscript we assume that the last coordinate of any point in the tropical projective torus $\mathbb R^d \!/\mathbb R {\bf   1}$ is equal to $0$ by taking this normalization.}

\begin{figure}[ht!]
\centering
\includegraphics[width=0.29\textwidth]{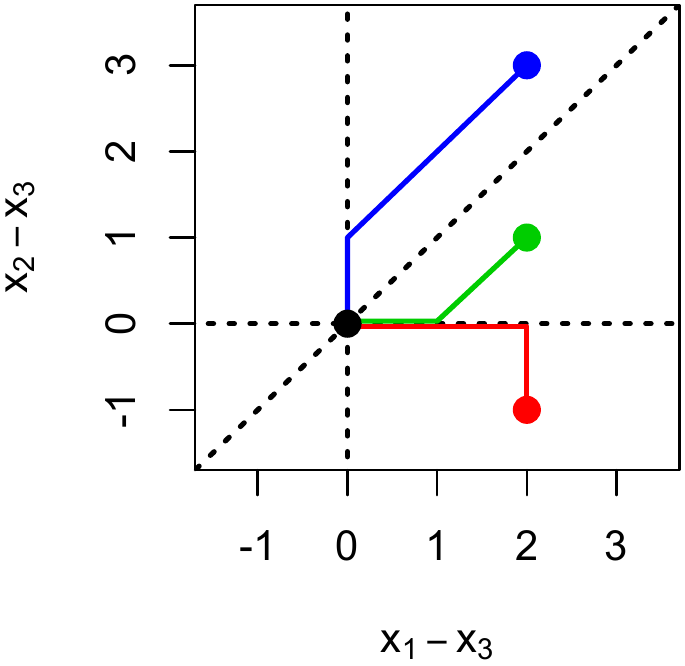}
\caption{{\color{black}
The tropical line segments between the origin and the three example points in $\mathbb R^3 \!/\mathbb R {\bf 1}$ are represented by the blue, green, and red lines.
The tropical distances from the origin to the three example points are $3$ (blue), $2$ (green), and $3$ (red).}}
\label{fig:D_tr}
\end{figure}

\begin{example}[{\color{black}Examples of Tropical Line Segments}]
{\color{black}Suppose we have $(2, 3, 0), \, (2, 1, 0), \, (2, -1, 0) \in \mathbb R^d \!/\mathbb R {\bf   1}$.  Then, the}
tropical line segments between the origin {\color{black} $(0, 0, 0) \in \mathbb R^d \!/\mathbb R {\bf   1}$ and these three points in $\mathbb R^d \!/\mathbb R {\bf   1}$ are drawn} by blue, green, and red lines in Figure \ref{fig:D_tr}.
\end{example}
}

\begin{definition}[Generalized Hilbert Projective Metric]
\label{eq:tropmetric} 
For any $v, \, w \in \mathbb R^d \!/\mathbb R {\bf 1}$ such that $v = (v_1, \ldots , v_d)$ and $w= (w_1, \ldots , w_d)$,  the {\em tropical distance} $d_{\rm tr}$ between them is defined such that:
\begin{equation*}
d_{\rm tr}(v,w)  := \max_{i} \bigl\{ v_i - w_i \bigr\} - \min_{i} \bigl\{ v_i - w_i \bigr\}.
\end{equation*}
\end{definition}

\begin{remark}
The tropical metric $d_{\rm tr}$ is a metric over $\mathbb R^d \!/\mathbb R {\bf 1}$.
\end{remark}

{\color{black}
\begin{example}
The tropical distances from the origin $(0,0,0)$ to the three example points in Figure\ref{fig:D_tr} in $\mathbb R^3 \!/\mathbb R {\bf 1}$ are
\[ d_{\rm tr}((2,3,0), (0,0,0)) = 3-0 = 3, \]
\[ d_{\rm tr}((2,1,0), (0,0,0)) = 2-{\color{black}0 = 2}, \]
\[ d_{\rm tr}((2,-1,0), (0,0,0)) = 2-(-1) = 3.\]
As one of many geodesics (shortest paths) toward the blue (green, red) point, the tropical line segment is denoted by the blue (green, red) line.
Intuitively, a tropical distance is a shortest path length as far as you are allowed to go only parallel to the dotted lines.
\end{example}
}

\begin{definition}[Tropical Hyperplane, \cite{Joswig}]
For any $\omega:=(\omega_1, \ldots, \omega_d)\in \mathbb R^d \!/\mathbb R {\bf 1}$, the {\em tropical hyperplane} defined by $\omega$ is the set of points $x\in \mathbb R^d \!/\mathbb R {\bf 1}$ such that 
{\color{black}
\[
\begin{array}{l}
H_{\omega}: = \Big\{(x_1, \ldots , x_d)\in \mathbb R^d \!/\mathbb R {\bf 1}|\exists i, j \in \{1, \ldots , d\}
\mbox{ such that }\\
\omega_i+x_i = \omega_j+x_j = \max \{\omega_1+x_1, \ldots \omega_d+x_d\}\Big\}.\\
\end{array} 
\]
} In addition, we call $\omega$ the {\em normal vector} of the tropical hyperplane $H_{\omega}$.
\end{definition}

{\color{black}
\begin{remark}
In terms of tropical geometry, $H_{\omega}$ is the solutions of the tropical linear function with unknown $x_1, \ldots , x_d$ such that for a fixed $\omega = (\omega_1, \ldots , \omega_d) \in \mathbb R^d \!/\mathbb R {\bf 1}$,
\begin{eqnarray}\nonumber
\omega_1\odot x_1 \boxplus \ldots \boxplus \omega_d \odot x_d \mbox{ for } x = (x_1, \ldots , x_d) \in \mathbb R^d \!/\mathbb R {\bf 1},
\end{eqnarray}
which is a special case of tropical polynomials. For more details on finding solutions of a tropical polynomial, see \cite{MS}.
\end{remark}

\begin{example}[Tropical Hyperplane in $\mathbb R^3 \!/\mathbb R {\bf 1}$]\label{example:thyperplane}
Suppose we have $\omega = (1, 2, 0) \in \mathbb R^3 \!/\mathbb R {\bf 1}$.  Then
{\color{black}
\[
\begin{array}{rcl}
   H_{\omega}: &=& \Big\{ (x_1, x_2, x_3) \in \mathbb R^3 \!/\mathbb R {\bf 1}|
   1+x_1 = 2+x_2 \geq x_3 \mbox{ or}\\
 && 1+x_1 = x_3 \geq 2+x_2 \mbox{ or }
       2+x_2 = x_3 \geq 1+x_1 \Big\}.\\
\end{array}
\]
}
This means that we have four cases:
\begin{enumerate}
    \item{\bf Case 1: The first term and the second term are equal and max.} $1+x_1 = 2+x_2 > x_3(=0)$.
    \item{\bf Case 2: The first term and the third term are equal and max.} $2+x_2 < 1+x_1 = x_3(=0)$.
    \item{\bf Case 3: The second term and the third term are equal and max.} $1+x_1 < 2+x_2 = x_3(=0)$.
    \item{\bf Case 4:  All terms are equal and max.} $1+x_1 = 2+x_2 = x_3(=0)$.
\end{enumerate}
We can set $x_3 = 0$ without loss of generality,
since $(x_1, x_2, x_3)$ and $(x_1 - x_3, x_2 - x_3, 0)$ represent the same point in $\mathbb R^3 \!/\mathbb R {\bf 1}$.
Thus $H_{(1, 2, 0)}$ is the set of $(x_1, x_2, 0)$ on the three half lines in Figure \ref{fig:thyper::ex}.  
\begin{figure}
    \centering
    \includegraphics[width=0.5\textwidth]{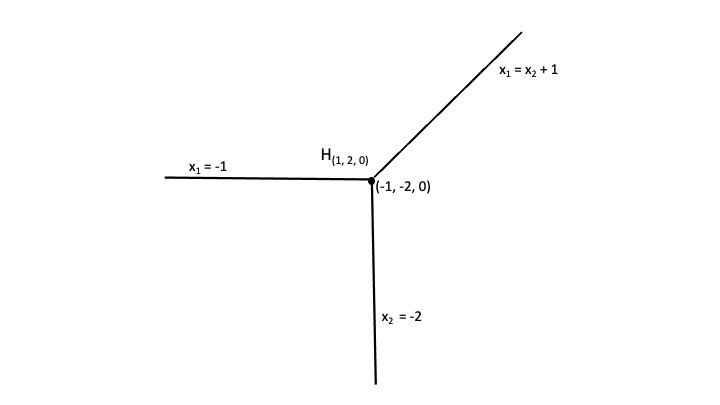}
    \caption{Tropical hyperplane defined by $\omega = (1, 2, 0)$ over $\mathbb R^3 \!/\mathbb R {\bf 1}$.}
    \label{fig:thyper::ex}
\end{figure}

\end{example}

}

\begin{definition}[Sectors of Tropical Hyperplane, \cite{Joswig}]
Each tropical hyperplane $H_{\omega}$ divides the tropical projective torus $\mathbb R^d \!/\mathbb R {\bf 1}$ into $n$ connected components.
These connected components 
are {\em open sectors} defined as:
$$S_{\omega}^i~:=~\{x\in \mathbb R^d \!/\mathbb R {\bf 1}\;|\; \omega_i+x_i>\omega_j+x_j,\;\forall j\neq i\;\},\;\;\!i\!=\!1,\!\ldots\!,d.$$
\end{definition}

\begin{definition}[Tropical Distance to a Tropical Hyperplane]
The {\em tropical distance} $d_{\rm tr}$ from a point $x\in \mathbb R^d \!/\mathbb R {\bf 1}$ to a tropical hyperplane $H_{\omega}$ is:
$$d_{\rm tr}(x, H_\omega)\;:=\;\min\{d_{\rm tr}(x, y)\;|\;y\in H_{\omega}\}.$$
\end{definition}

\begin{proposition}[Lemma 2.1 in \cite{Gartner}]\label{pp:phdis}
Let $H_{\bf 0}$ be the tropical hyperplane defined by the zero vector ${\bf 0} = (0, 0, \ldots , 0) \in {\mathbb R}^d \!/\mathbb R {\bf 1}$.
For any $x=(x_1, \ldots, x_d) \in \mathbb R^d \!/\mathbb R {\bf 1}$,
\[
d_{\rm tr}(x, H_{\bf 0}) = {\rm max}(x) - {\rm second \; max}(x). 
\]
\end{proposition}

\begin{corollary}[Corollary  2.3 in  \cite{Gartner}]\label{cry:phdis}
For any $\omega\in {\mathbb R}^d \!/\mathbb R {\bf 1}$ and for any $x\in \mathbb R^d \!/\mathbb R {\bf 1}$, 
\[
d_{\rm tr}(x, H_\omega) = d_{\rm tr}(\omega+x, H_{\bf 0}).
\]
\end{corollary}

{\color{black}
\begin{example}\label{example:opensector}
We use the sample $H_{\omega}$ from Example \ref{example:thyperplane}.  Let $d = 3$ and suppose $\omega = (1, 2, 0) \in \mathbb R^3 \!/\mathbb R {\bf 1}$.  
In addition, suppose we have a point $x=(1, 1, 0)\in \mathbb R^3 \!/\mathbb R {\bf 1}$. By Corollary \ref{cry:phdis}, 
$$d_{\rm tr}(x, H_{\bf \omega}) = d_{\rm tr}(x+\omega, H_{\bf 0}) = d_{\rm tr}((2, 3, 0), H_{\bf 0}) =  3 - 2 =1.$$ 
\end{example}}

\begin{figure}[t!]
\centering
\includegraphics[width=0.215\textwidth]{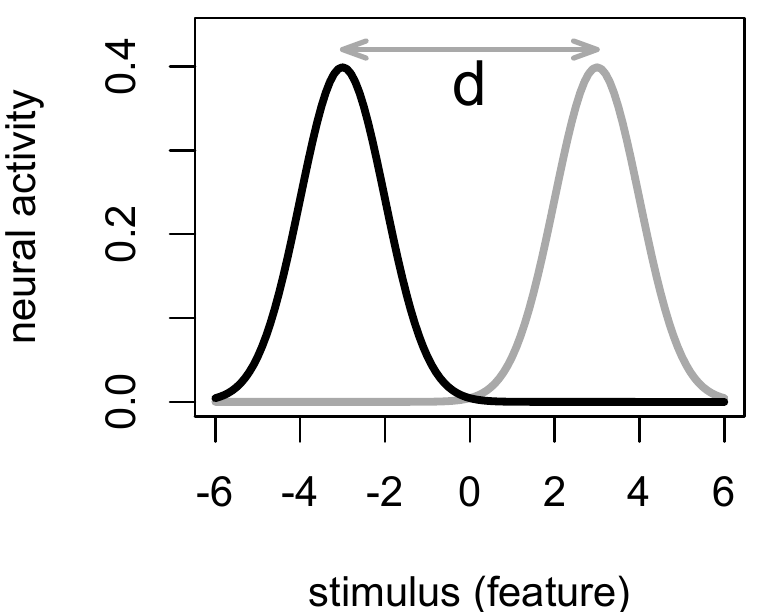} \ ~ 
\includegraphics[width=0.215\textwidth]{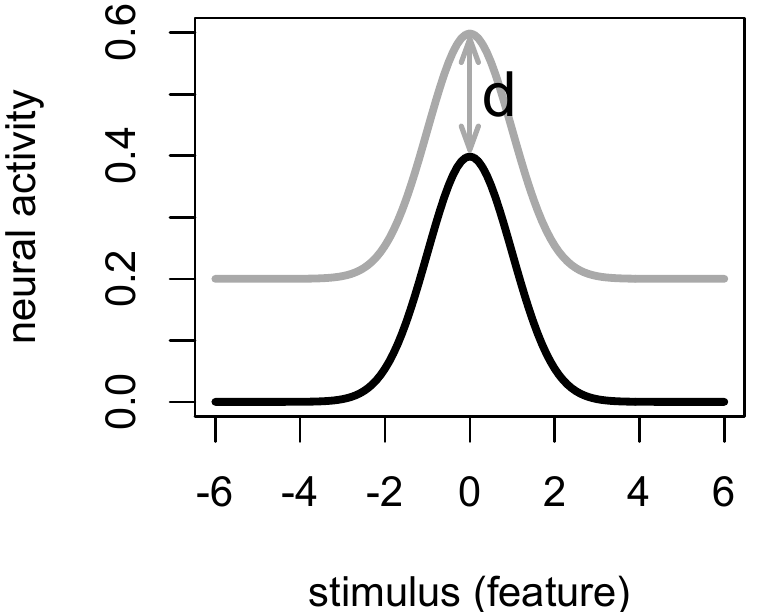} \newline \newline 
\includegraphics[width=0.215\textwidth]{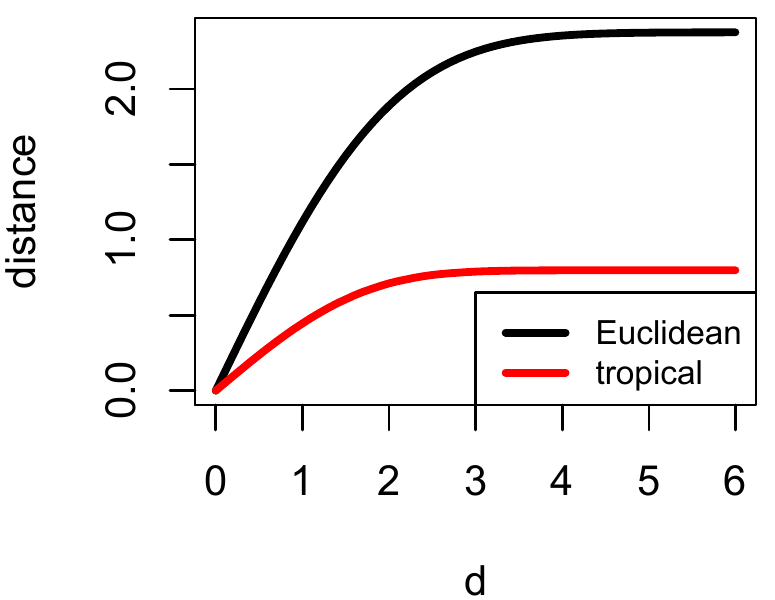} \ ~ 
\includegraphics[width=0.215\textwidth]{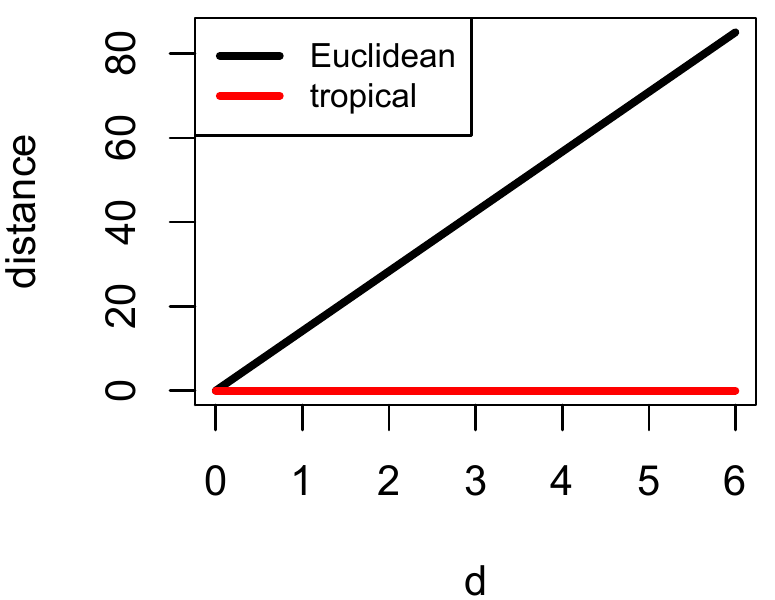} \newline
\caption{Tropical and Euclidean distances between two bell-shaped tuning curves which are horizontally shifted (left column) or vertically lifted (right column). The tropical distance is insensitive to the vertical lift of tuning curves.}
\label{fig:distBellShapedTC}
\end{figure}

\section{Application 1: Valuable Scaling Behaviors of Tropical Distances}\label{sec:app1}

\subsection{Bell-shaped Tuning Curves}
To characterize the properties of the tropical distances $d_{\rm tr}(v,w)$ proposed in this paper,
we computed the tropical distances for the mathematical models of neuronal tuning curves as benchmarks.
The tuning curves are the vectors consisting of neural responses to different stimuli.
For example, the responses of neurons in the primary visual cortex to oriented bars on the screens as visual stimuli are known to show famous bell-shaped curves as a function of the orientation.

We modeled neuronal tuning curves basically as standard Gaussian distribution functions and examined how the distances scale when one of two curves is horizontally shifted or vertically lifted.
Both the tropical and Euclidean distances between two bell-shaped tuning curves increased with the increasing horizontal shift between them (Figure~\ref{fig:distBellShapedTC}, left).
Meanwhile, the tropical distances between the two bell-shaped tuning curves was insensitive to the vertical lift of tuning curves, while the Euclidean distance was not (Figure~\ref{fig:distBellShapedTC}, right).
The insensitivity to the vertical shift originates from the definition of the tropical distances $d_{\rm tr}(v,w)$ where tuning curves are considered in $\mathbb R^d \!/\mathbb R {\bf 1}$ and $(1, 1, \ldots, 1)$-direction is collapsed and ignored in Def.\ref{eq:tropmetric}. That is, $d_{\rm tr}(v+c{\bf 1},w+d{\bf 1}) = d_{\rm tr}(v,w)$ for any $c$ and $d$ in $\mathbb R$.
\begin{figure}[t]
\centering
\includegraphics[width=0.31\textwidth]{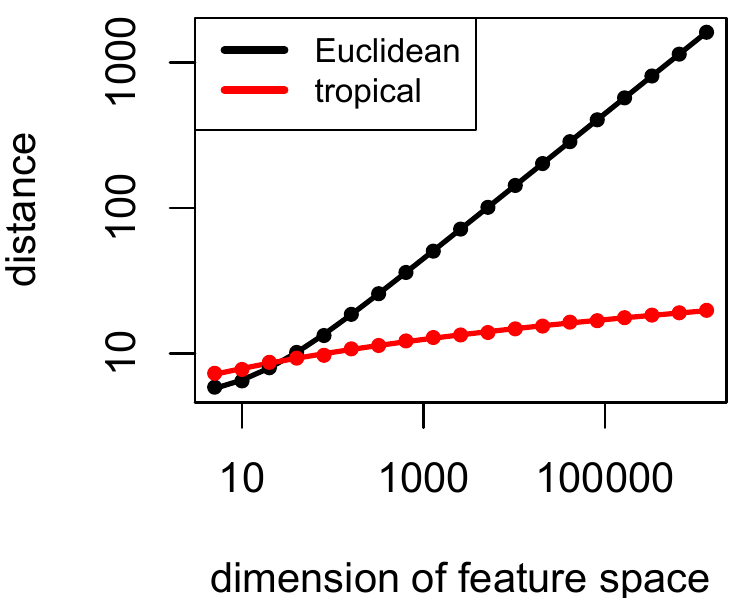}
\caption{Tropical and Euclidean distances between random and flat tuning curves. The tropical distance scales with $\log n$ as a result of extreme value statistics. The circles denote simulation results and the lines denote theoretical results.}
\label{fig:distRandomTC}
\end{figure}
Note that the entire lifting of tuning curves are often observed as a result of drifting background neural activities (\cite{harris15,harris19,miura05,miura06,miura07,miura13,miura19}).
Thus, practically, the tropical distance can be useful for the sake of ignoring the temporal drift of tuning curves.

\subsection{Random Tuning Curves}\label{sec:app1_rand}
Next we considered a random tuning curve consisting of noisy responses. 
When we computed the distances between a random tuning curve
\begin{equation}
v = (5, -x_2, -x_3, \ldots, -x_n) ,
\end{equation}
where $x_i \sim \textrm{Exp(1)}$ and a flat tuning curve,
\begin{equation}
w = (0, 0, 0, \ldots, 0) ,
\end{equation}
the tropical distances was relatively insensitive to the additional noise dimensions.
That is, the tropical distance scaled with $\log n$ while the Euclidean distance scaled with $\sqrt{n}$.
The anomalous scaling for the tropical distance can be explained by the extreme value statistics (\cite{gumbel04}) as,
\begin{eqnarray}
d_{\rm tr}(v,w) &:=& \max_{i} \bigl\{ v_i - w_i \bigr\} - \min_{i} \bigl\{ v_i - w_i \bigr\}\nonumber \\
&\approx& 5 + \max_{i} \bigl\{ x_i \bigr\}.
\end{eqnarray}
Then, because $ \max_{i} \bigl\{ x_i \bigr\} -\log n \sim \textrm{Gumbel(0,1)}$ (\cite{gumbel04}), the expectation is given as
\begin{equation}
\overline{d_{\rm tr}(v,w)} = 5 + \gamma + \log n .
\end{equation}
This theoretical result explains the simulation result in Figure 4.
In this way, the extreme value statistics plays the crucial role in the computation of the tropical distance between random vectors.

\section{Evaluation of Tropical Support Vector Machines (SVMs)}\label{sec:svm}

\subsection{$L_2$-norm SVMs}

In this section we provide an overview of the $L_2$ norm SVMs.  
Let $\mathcal{D}_2$ be the distribution on the joint random variable $(X, Y)$ for $X \in \R^d$ and $Y \in \{-1, 1\}$ and let $\mathcal{S}_2$ be the sample $\mathcal{S}_2:= \{(X^1, Y^1), \ldots , (X^n, Y^n)\}$.
Then the standard linear SVM is a statistical model by solving the following regularization problem:
\[
\min_{\omega}\left(\underbrace{\lambda ||\omega||^2}_\text{regularizer}+\underbrace{n^{-1}\sum_{i = 1}^n (1 - Y^i ((X^i)^T \omega + \omega_0))_+}_\text{error}\right),
\]
where $||x||$ is the $L_2$ norm of a vector $x \in \R^d$, $(1-u)_+ = \max\{1-u, 0\}$ is the hinge loss function, $\lambda$ is a tuning parameter, $\omega_0$ is a constant term of the hyperplane, and  $\omega$ is the normal vector of the hyperplane to separate the data points.
\cite{Vapnik} discussed the generalization bounds on the $L_2$ norm SVMs using Theorem \ref{bound0} and the {\em VC dimension} of the $L_2$ norm SVMs.  In addition, using the Rademacher complexity of $L_2$ norm of unit vectors (\cite{10.5555/944919.944944}), one can show the following generalization bound of the error rate:
\begin{theorem}
Let $w^s$ be an output from the classical hard margin SVM from a sample $\mathcal{S}$.  Here $Y \in \{-1, 1\}$ with the distribution $\pi_+ = P(Y = 1)$ and $\pi_- = P(Y = -1)$.  Let $X \in \RR^d$ be a random variable and $\mathcal{D}_{2}$ be the distribution for a random variable $(X,  Y)$.  Let $||x||$ be the $l_2$ norm of a vector $x \in \RR^d$. Assuming that $||X|| \leq R$ with probability 1 and there exists $\omega^*$ with $P_{\mathcal{D}_2}(Y^i ((X^i)^T \omega + \omega_0) \geq 1) = 1$.
Then, for any $\eta > 0$ with the probability greater than or equal to $1 - \eta$ we have
{\small
\begin{equation}\label{bound2}
P_{\mathcal{D}_2}\left(Y \not = sign(X \cdot \omega^s) \right) \leq \frac{4R||\omega^s||}{\sqrt{n}}+ (1+2R||\omega^s||) \sqrt{\frac{2\log(4||\omega^s||/\eta)}{n}}.
\end{equation}}
\end{theorem}

There has been much work done in tighter generalization bounds on the $L_2$ norm SVMs using the {\em VC dimension}.
As it is not this paper's focus to discuss the details on the generalization bounds for the $L_2$ norm SVMs, for more details, see \cite{Burges,Guermeur} and the references therein.

\begin{remark}
Recall that the $L_p$ norm SVMs, for $0 \leq p \leq \infty$, can be written as the following optimization problem:
\[
\min_{\omega}\left(\underbrace{\lambda ||\omega||_p^2}_\text{regularizer}+\underbrace{n^{-1}\sum_{i = 1}^n (1 - Y^i ((X^i)^T \omega + \omega_0))_+}_\text{error}\right),
\]
where $||x||_p$ is the $L_1$ norm of a vector $x \in R^d$, $(1-u)_+ = \max\{1-u, 0\}$ is the hinge loss function, $\lambda$ is a tuning parameter, and $\omega$ is the normal vector of the hyperplane to separate the data points.
\end{remark}

\subsection{Tropical SVMs}

In this section we provide an overview the hard margin tropical SVMs {\color{black} originally developed by \cite{Gartner}} with the random variable $X \in \mathbb R^d \!/\mathbb R {\bf 1}$ given the response variable $Y \in \{0, 1\}^d$ {\color{black}
for one-hot encoding. Note that tropical SVMs can perform multiclass classification naturally}.
Before we formally define the hard margin tropical SVMs, we need to define some notation.

Let $S(x)$ be a set of indices of nonzero elements in  a vector $x \in \mathbb R^d \!/\mathbb R {\bf 1}$, i.e., $S(x) \subset \{1, \ldots , d\}$ where $x_i \not = 0$.
Let $I_{\omega}(x) \in \{0, 1\}^d$ be a vector of indicator functions of index set $\{1, \ldots , d\}$ of a vector $x$ with a vector $\omega \in \RR^d \!/\mathbb R {\bf 1}$ where 
\[
I_{\omega}(x)_i = \begin{cases}
1 & \text{if } x_i + \omega_i = \max(x + \omega)\\
0 & \text{otherwise.}
\end{cases}
\]
Let $J_{\omega}{\color{black}(x)} \in \{0, 1\}^d$ be also a vector of indicator functions of index set $\{1, \ldots , d\}$ of a vector $x$
 such that 
\[
J_{\omega}(x)_i = \begin{cases}
1 & \text{if } x_i + \omega_i = \text{second max}(x + \omega)\\
0 & \text{otherwise.}
\end{cases}
\]

{\color{black}
Suppose we have a categorical response variable $Y' \in \{g_1, \ldots , g_q\}$ with $q \leq d$ many levels.  Then we set a vector of indicator functions $Y = (Y_1, \ldots , Y_d)$ such that
\[
Y_i = \begin{cases}
1 & \mbox{if } Y' = g_i\\
0 & \mbox{else.}
\end{cases}
\]

\begin{example}
Suppose that we have a binary categorical response variable $Y' \in \{\mbox{"yes" , "no"}\}$.  Then we can set as 
a vector of indicator functions $Y = (Y_1, \ldots , Y_d)$ such that
\[
Y_1 = \begin{cases}
1 & \mbox{if } Y' = \mbox{yes}\\
0 & \mbox{else,}
\end{cases}
\]
\[
Y_2 = \begin{cases}
1 & \mbox{if } Y' = \mbox{no}\\
0 & \mbox{else,}
\end{cases}
\]
and
\[
Y_3 = Y_4 = \ldots = Y_d = 0.
\]
\end{example}

In order to simplify the problem, here we consider a case that there are only two classes $Y^A$ and $Y^B$ in the response variable.
More precisely, we have a random variable $Y^A := (Y^A_1, \ldots, Y^A_d), Y^B := (Y^B_1, \ldots, Y^B_d) \in \{0, 1\}^d$ such that
for fixed $i, k \in \{1, \ldots , d\}$ with $ i \neq k$,
\[
Y^A_j = \begin{cases}
1 &\mbox{if } j = i\\
0 &\mbox{else,}
\end{cases}
\]
and
\[
Y^B_l = \begin{cases}
1 &\mbox{if } l = k\\
0 &\mbox{else}
\end{cases}
\]
with a discrete probability $\pi_A := P(Y = Y^A)$ and $\pi_B := P(Y = Y^B)$
}.
Then, suppose we have a multivariate random variable $X \in \mathbb R^d \!/\mathbb R {\bf 1}$ given $Y$ with the probability density function
{\color{black}
$f_A$ if $Y = Y^A$
}
and the probability density function
{\color{black}
$f_B$ if $Y = Y^B$
}
such that there exists a tropical hyperplane $H_{\omega^*}$ with a normal vector $\omega^* \in \mathbb R^d \!/\mathbb R {\bf 1}$ with the following properties:
\begin{itemize}
\item[(i)] there exists an index $i\in \{1, \ldots, d\}$ such that
\[ \text{for any}\; j\in \{1, \ldots, d\}\backslash\{ i\},\;\; \omega^*_{i}+X_{i}\;>\;\omega^*_{j}+X_{j}, \;\;\;\text{and}\]
\item[(ii)] 
$$\max \left\{I_{\omega^*}(X) - Y \right\} = 0,$$
\end{itemize}
with probability $1$.

Let $\mathcal{D}$ be the distribution on the joint random variable $(X, Y)$ and let $\mathcal{S}$ be the sample $\mathcal{S}:= \{(X^1, Y^1), \ldots , (X^n, Y^n)\}$.
Then we formulate an optimization problem for solving the normal vector $\omega$ of an optimal tropical separating hyperplane $H_{\omega}$ for random variables $X$ given $Y$: For some cost $C \in \RR$
{\tiny
\begin{equation}\label{equation:24}
\begin{matrix}
\displaystyle  \left[
 \max \limits_{\omega \in \RR^d \!/\RR {\bf 1}}\; \min \limits_{X \in \mathcal{S},  i \in S(I_{\omega}(X)), j \in S(J_{\omega}(X))}\left\{
\underbrace{\left(X_i+\omega_{i}-X_{j}-\omega_{j}\right)}_\text{margin}+ \underbrace{ \frac{C}{n} \sum_{k=1}^n\min\left\{Y^k - I_{\omega}(X^k) \right\}}_\text{error}\right\}\right]. \\
\end{matrix}
 \end{equation}}
Here, the expectation of the random variable $\max \left\{I_{\omega}(X) - Y \right\}$ is the $0-1$ loss function.  
Also note that
$$d_{\rm tr}(X, H_\omega)\;=\;X_{i}+\omega_{i}-X_{j}-\omega_{j},$$ where $i \in S(I_{\omega}(X)), j \in S(J_{\omega}(X))$.
Thus, this optimization  problem can be explicitly written as a linear programming problem \eqref{equation:251}--\eqref{equation:254} below, where the optimal solution $z$ means the {\em margin} of the tropical SVM: For some cost $C \in \RR$
{\scriptsize
    \begin{align}
  & \max \limits_{(z, \omega) \in \mathbb{R} \times \mathbb R^d \!/\mathbb R {\bf 1}} \; \left(z + \frac{C}{n} \sum_{k=1}^n \min \left\{Y^k - I_{\omega}(X^k) \right\}\right) \label{equation:251} \\
  \textrm{s.t.}\;\; \forall X &\in \mathcal{S},  \forall i \in S(I_{\omega}(X)), \forall j \in S(J_{\omega}(X)) ,  \;\;z+\textcolor{black}{X_{j}}+\omega_{j}\textcolor{black}{-X_{i}}-\omega_{i}\leq 0,  \label{equation:252}\\
  \forall X &\in \mathcal{S}, \forall i \in S(I_{\omega}(X)), \forall  j \in S(J_{\omega}(X)) , \;\; \omega_{j}-\omega_{i}\leq X_{i}-X_{j}, \label{equation:253} \\
  \forall X &\in \mathcal{S}, \forall l \not \in  S(I_{\omega}(X))\cup S(J_{\omega}(X)), j \in S(J_{\omega}(X)) \;\; \omega_{l}-\omega_{j}\leq X_{j}-X_{l}. \label{equation:254}
  \end{align}
}
\begin{remark}
With the regularization with the tropical metric to the tropical hyperplane, it minimizes the coefficients of the tropical hyperplane $\omega_i$ for $i = 1, \ldots d$ and tries to minimize the difference between the coordinate of the maximum and second maximum of $X + \omega$.  This means that this regularization behaves similar to the $L_1$ or $L_{\infty}$ norm regularization, that is, to select features which contribute to discriminate
{\color{black}
$X|Y^A$
} and
{\color{black}
$X|Y^B$
}.
\end{remark}

\subsection{Bound on Generalization Errors via VC Dimensions}\label{sec:main} 
In this section we show the generalization bound for the error rate of the hard margin tropical SVM. In order to compute them, we define the loss function in terms of the distribution of $X|Y$ defined in the previous section and the sample loss function. In order to prove the generalization bound, we use the {\color{black} VC dimension \citep{vapnik95,MohriRostamizadehTalwalkar18}}  

Let $\mathcal{D}$ be the distribution of $X|Y$ defined in Section \ref{sec:svm}. In addition, let
\[
L_{\mathcal{D}}(\omega) =  \mathbb{E}_{\mathcal{D}}\left(\max \left\{ I_{\omega}(X) - Y \right\}\right)
\]
be the loss function with respect to $\mathcal{D}$ and let
\[
L_{\mathcal{S}}(\omega) =
\mathbb{E}_{\mathcal{S}}\left(\max \left\{ I_{\omega}(X) - Y \right\}\right)
\]
be the loss function with respect to the sample $\mathcal{S}$.

Let $(Z_s, \omega^s)$ be an output of the tropical SVM. Then we want to find the upper bound for 
\begin{equation}\label{prob1}
P\left(\max \left\{I_{\omega^s}(X)-Y \right\} > 0 \right) = L_{\mathcal{D}}(\omega^s).
\end{equation}

In order to prove our bound, we have to define {\em VC dimension}. 
\begin{definition}
Suppose $M$ is a set and a family of a subset $R \subset 2^M$.  A set $T \subset M$ is called {\em shattered} by $R$ if 
\[
|\left\{r \cap T| r \in R\right\}| = 2^{|T|}. 
\]
The pair $(M, R)$ is called {\em range space}.
The {\em VC dimension} of a range space, $VC-dim(M,R)$ is the maximal cardinality of a subset of $M$ which can be shattered by $R$.
\end{definition}

{\color{black}
\begin{example}[Linear Classifiers]
In $(d-1)$-dimensional Euclidean space, a linear classifier can shatter or successfully classify $d$ points in general positions with arbitrary labels.
Thus, the VC dimension of the linear classifier is $d$.
\end{example}

\begin{lemma}[Tropical Radon Lemma in \cite{Jaggi_newresults}]\label{lm:Radon}
For $d+1$ points in $\mathbb R^d \!/\mathbb R {\bf 1}$, there exists a partition into two sets with intersecting convex hulls, which are therefore not separable by any tropical hyperplane $H_{\omega}$.
\end{lemma}

\begin{example}
Any tropical hyperplane in $\RR^3 \!/\mathbb R {\bf 1}$ cannot shatter the four points shown in Figure \ref{fig:vc}.
The fact that the two line segments intersect to each other suggests that it is impossible to classify them by tropical hyperplanes.
In fact, ANY four points cannot be shattered by a tropical hyperplane in $\RR^3 \!/\mathbb R {\bf 1}$.
This is because, given four or more points, there always exist ways of partitioning so that the tropical convex hull of the points in each class intersects.
\label{ex:convexhull}
\end{example}

\begin{figure}[t!]
\centering
\includegraphics[width=0.27\textwidth]{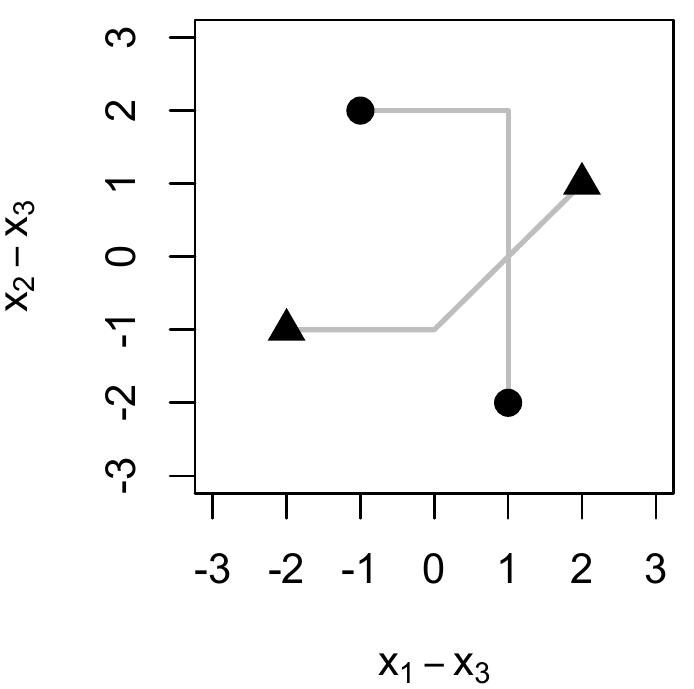}
\caption{{\color{black}
Example positions where four points cannot be shattered by tropical hyperplanes in $\RR^3 \!/\mathbb R {\bf 1}$. The gray lines represent the tropical convex hulls of the two points in the same class.
The fact that the line segments intersect suggests that it is impossible to classify them by tropical hyperplanes.}}
\label{fig:vc}
\end{figure}

\begin{lemma}[Lemma 27 in \cite{Jaggi_newresults}]\label{lm:vc}
There exists a set of $d$ points in $\mathbb R^d \!/\mathbb R {\bf 1}$ that can be shattered by single sectors defined by a tropical hyperplane $H_{\omega}$.
\end{lemma}

While Tropical Radon Lemma \ref{lm:Radon} shows that the VC dimension of a tropical hyperplane is smaller than $d+1$,
(trivial) Lemma \ref{lm:vc} shows that it is larger than or equal to $d$
(as single sectors should have less ability than hyperplanes).

\begin{theorem}[\cite{Jaggi_newresults}]\label{thm:vc}
Let $\mathbf{H}$ be a family of tropical hyperplanes in $\RR^d \!/\mathbb R {\bf 1}$.  Then the VC dimension of a range space $(\TPP , \mathbf{H})$ is $d$.
\end{theorem}
}

\begin{theorem}[\cite{Vapnik}]\label{bound0}
Let $\mathcal{H}$ be a family of classification models
{\color{black}whose VC dimension is $h$}.
Let $n$ be the sample size of the sample $\mathcal{S}$. Then, for any $\eta > 0$ and for a sufficiently large sample size $n$ with the probability greater than or equal to $1 - \eta$, we have 
\[
\ell_{\mathcal{D}'}(\mathcal{H}) - \ell_{\mathcal{S}}(\mathcal{H}) \leq  \sqrt{\frac{h (\log(2n/h) + 1) - \log({\color{black}\eta/4})}{n}},
\]
where $\ell_{\mathcal{D}'}(\mathcal{H})$ is the $0-1$ loss for the distribution $\mathcal{D}'$ with $\mathcal{H}$ and $\ell_{\mathcal{S}}(\mathcal{H})$ is the $0-1$ loss for the sample $\mathcal{S}$ with $\mathcal{H}$.
\end{theorem}

{\color{black}
The generalization error bound for tropical SVMs is obtained by plugging in the VC dimension of tropical hyperplanes in Theorem \ref{thm:vc} for $h$ in Theorem \ref{bound0}.
}

\begin{theorem}
Let $n$ be the sample size of the sample $\mathcal{S}$. Then, for any $\eta > 0$ and for a sufficiently large sample size $n$ with the probability greater than or equal to $1 - \eta$, we have 
{\small
\begin{equation}\label{bound1}
P_{\mathcal{D}}\left(\max \left\{I_{\omega^s}(X) - Y \right\} > 0 \right) \leq  \sqrt{\frac{d (\log(2n/d) + 1) - \log({\color{black}\eta/4})}{n}}.
\end{equation}}
\end{theorem}

\begin{proof}
Let $\mathcal{S}$ be a sample for a test set.  Since there is a tropical hard margin and $(z_s, \omega^s)$ is a feasible solution on the test set $\mathcal{S}$, $L_{\mathcal{S}}(\omega^s) = 0$.  
With Theorem \ref{thm:vc} and Theorem \ref{bound0}, we have
\[
L_{\mathcal{D}}(\omega^s) - L_{\mathcal{S}}(\omega^s) \leq \sqrt{\frac{d (\log(2n/d) + 1) - \log({\color{black}\eta/4})}{n}}.
\]
Since $L_{\mathcal{S}}(\omega^s) = 0$ by the assumption and 
\[
L_{\mathcal{D}}(\omega^s) = P_{\mathcal{D}}\left(\max \left\{I_{\omega^s}(X) - Y \right\} > 0 \right),
\]
 we have the result.  
\end{proof}


\section{Application 2: Robustness of Tropical SVMs against Curse of Dimensionality}\label{sec:comp}
In the previous section, we derive the generalization error bound for tropical SVMs.
The bound is stated in a distribution-free manner as a general feature applicable to all data.
Specifically, the bound is derived purely combinatorially only based on the shapes of the hyperplanes.
Algorithmic details of tropical SVMs are not yet taken into account there.
Therefore it is possible that tropical SVMs outperform in specific situations, for example, when only a few features are prominently informative among many.
Thus the evaluation of the performance of tropical SVMs by computational experiments can give complementary information.
Here we applied tropical SVMs to specific problems of computational experiments and real data analyses to evaluate its performance originating from the max-plus algebra.

\subsection{Implementation of Tropical SVMs}
Soft margin tropical SVMs are defined in \cite{tang} which allow some data points in the wrong open sectors with the hinge loss function as similar to the $L_2$ norm SVMs. Since the loss function deals with $L_0$ norm of $I_{\omega}(X) - Y$, we have to run exponentially many linear programming problems in order to find the optimal tropical hyperplane to fit the data set.  Therefore, Tang et al.~introduced several heuristic algorithms to estimate the optimal tropical hyperplane to fit the data: Algorithms 1 -- 4.
In the computational experiments in this paper, we applied Algorithms 3, the simplest one, from \cite{tang} to fit each data to the model, as it always performed the best for our simple examples.
We obtained the {\tt R} code from \url{https://github.com/HoujieWang/Tropical-SVM}.

\subsection{Computational Experiments}
As we observe in Section \ref{sec:app1_rand}, the tropical distance can ignore the nuisance dimensions that are purely noises.
Therefore it is expected that tropical SVMs can show a good performance in the curse-of-dimensionality regime where extra dimensions are all uninformative.

Here we performed a simple computational experiment of the two-class classification using the uncorrelated unit-variance Gaussian distributions whose means are
${\color{black}m_1 =}$ $(s, -s, 0, 0, \ldots, 0)$ for class $1$ and ${\color{black}m_2 =}$ $(-s, s, 0, 0, \ldots, 0)$ for class $2$, where $s=\sqrt{2}$ or $5$.
{\color{black}
To be specific, the unit covariance matrices are common for the two distributions as $\Sigma_1 = \Sigma_2 =
\begin{pmatrix}
1 & 0 & \hdots & 0\\
0 & 1 & \hdots & 0\\
\vdots & \vdots & \vdots & \vdots\\
0 & 0 & \hdots & 1\\
\end{pmatrix}$.
}
Note that the signal or the difference of the means{\color{black}, $m_2 - m_1 = (2s, -2s, 0, 0, \ldots, 0)$, is} orthogonal to the irrelevant direction $(1, 1, 1, 1, \ldots, 1)$.
{\color{black}As a measure of separation, the d-prime or the Euclidean distance between $m_1$ and $m_2$ divided by the S.D. (=1) is given by $d' = 2s\sqrt{2}$.}

{\color{black}
Before we look at the curse-of-dimensionality regime or small samples, we first consider the case when the numbers of training and test samples $N$ are fairly large
to compare the numerically obtained generalization errors with the theoretical bounds in Figure \ref{fig:VCbound}.
We draw the lower bound of the classification success rates by subtracting
the penalty term in the right hand side of the inequality in Theorem \ref{bound0}
from the hit rate for the training set.
The gaps between the numerical computations and the theoretical bounds suggest that the bounds are not so tight.
In fact, the lower bounds are mostly smaller than the chance level $50\%$ and thus not informative.
Although we solely use $\eta = 0.1$, the bounds are rather insensitive to the value of $\eta$.
Even if we have as large as $N=1000$ or $10000$ observations, the gap is still quite noticeable as expected from
the penalty term in the right hand side of the inequality in Theorem \ref{bound0}.

\begin{figure}[t!]
\centering
\includegraphics[width=0.215\textwidth]{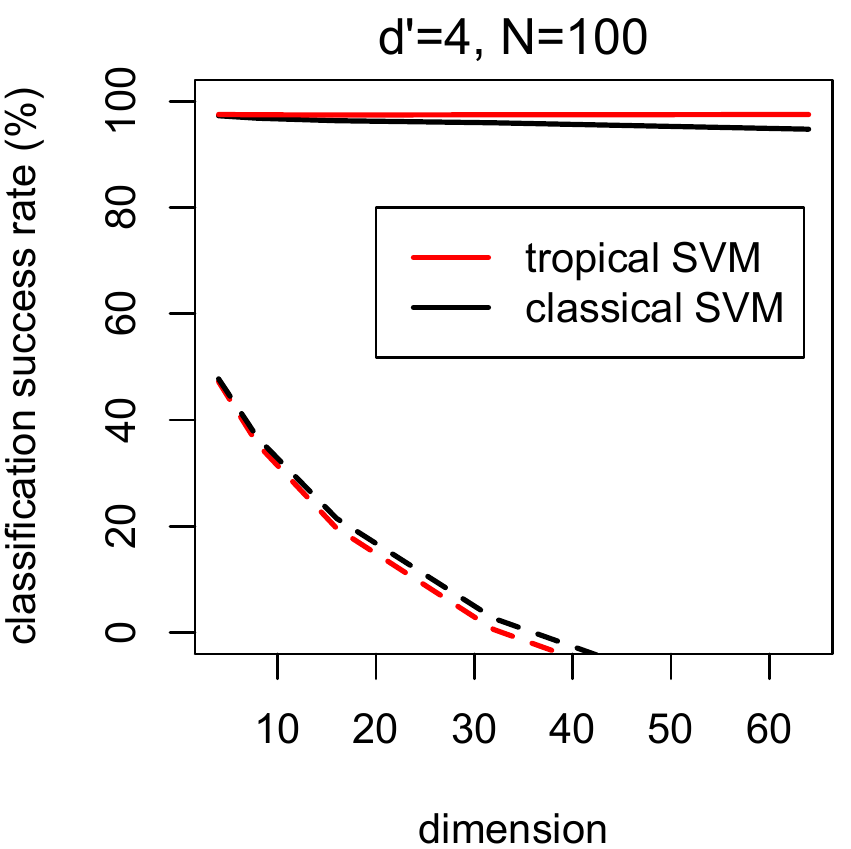} \ ~ 
\includegraphics[width=0.215\textwidth]{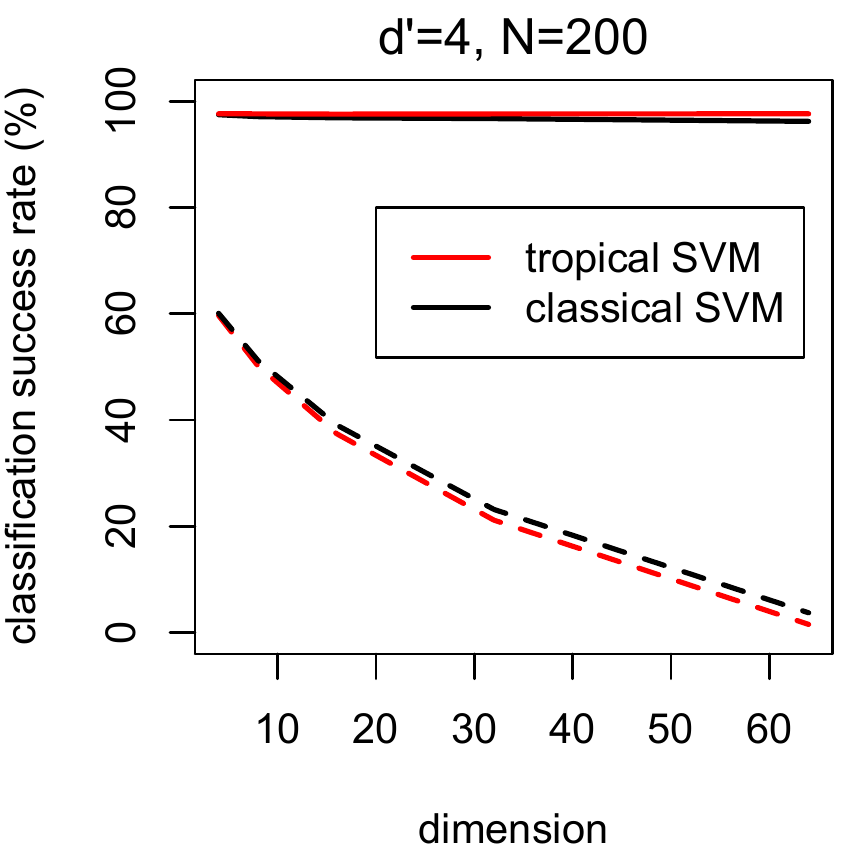}
\caption{{\color{black}
Numerically obtained classification success rates (thick lines) and theoretical lower bounds (dotted lines) for the classification of two multivariate Gaussian distributions
{\color{black}whose means are separated by $d'=4$ and covariances are the unit matrices.}
The results for tropical SVMs are red colored while those for classical SVMs are black colored. The numbers of training and test samples for each class are $100$ (left column) and $200$ (right column). The cross-validated performance averaged over $1000$ realizations of data were plotted. Note that the lower bounds of the hit rates (=upper bounds of error rates) are not tight.}}
\label{fig:VCbound}
\end{figure}

Next, we consider small samples as a curse-of-dimensionality regime.}
Figure \ref{fig:tropSVMgaussian} shows the classification success rates of two multivariate Gaussian distributions separated by ${\color{black}d'= }4$ (top) or ${\color{black}d'= }10\sqrt{2}$ (bottom).
The numbers of training and test samples for each class are 5 (left) and 10 (right).
In all four cases, the hit rate decreases significantly for classical SVMs, while it decreases only slightly for tropical SVMs.
In fact, it almost stays constant for the ${\color{black}d'=}10\sqrt{2}$ cases.
Thus the tropical SVM is robust against the curse of dimensionality in the case for separating two Gaussian distributions.

\begin{figure}[t!]
\centering
\includegraphics[width=0.215\textwidth]{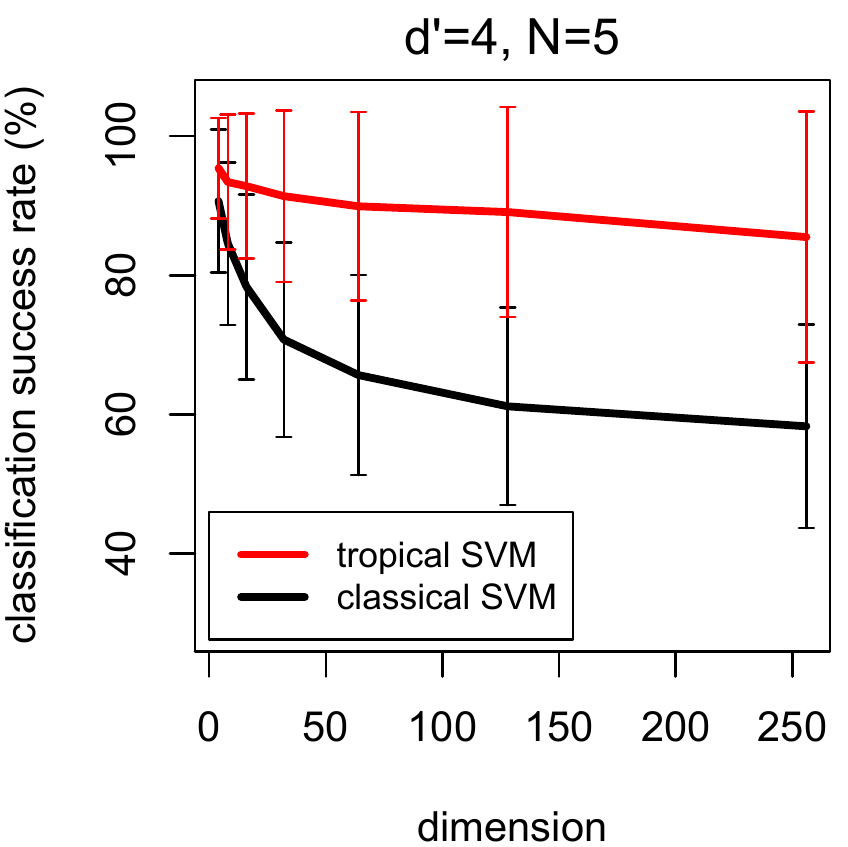} \ ~ 
\includegraphics[width=0.215\textwidth]{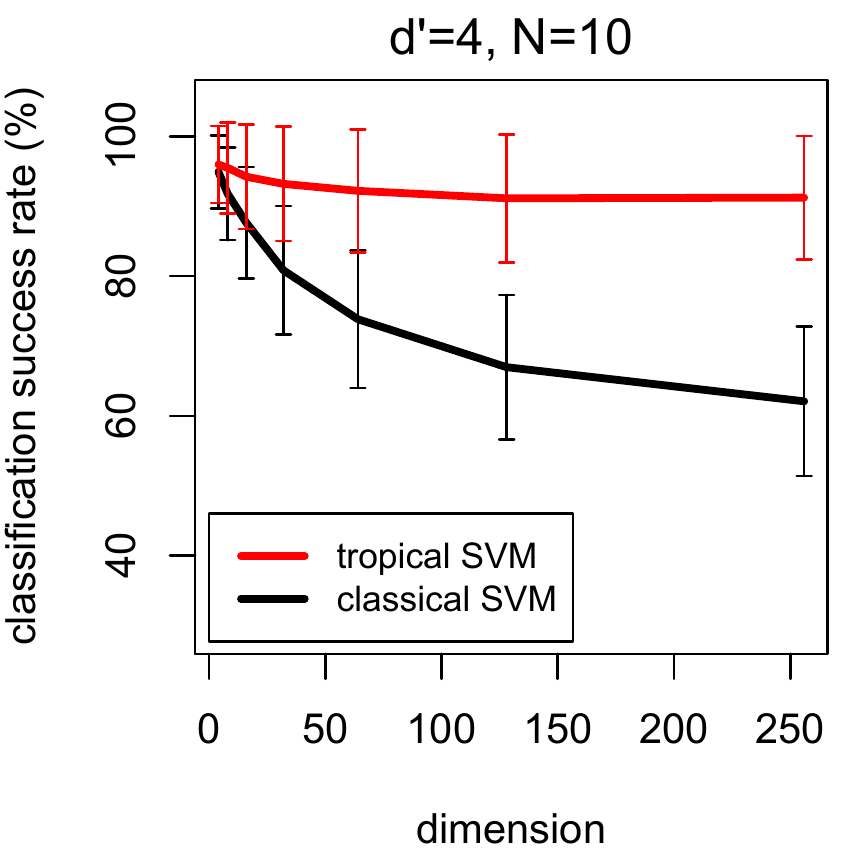} \newline \newline
\includegraphics[width=0.215\textwidth]{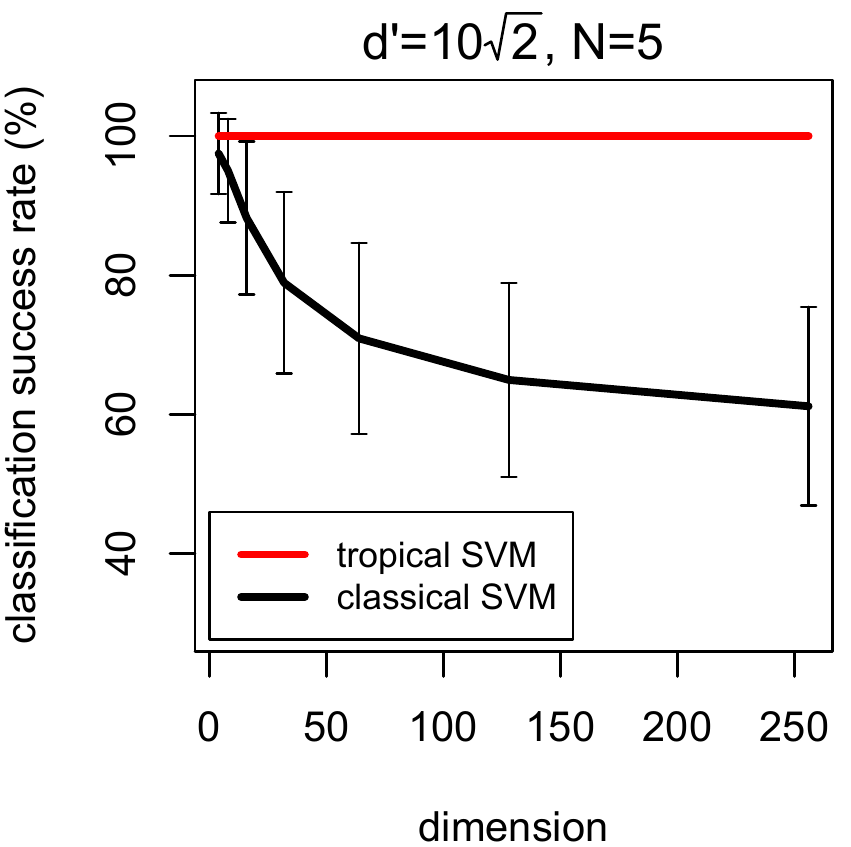} \ ~ 
\includegraphics[width=0.215\textwidth]{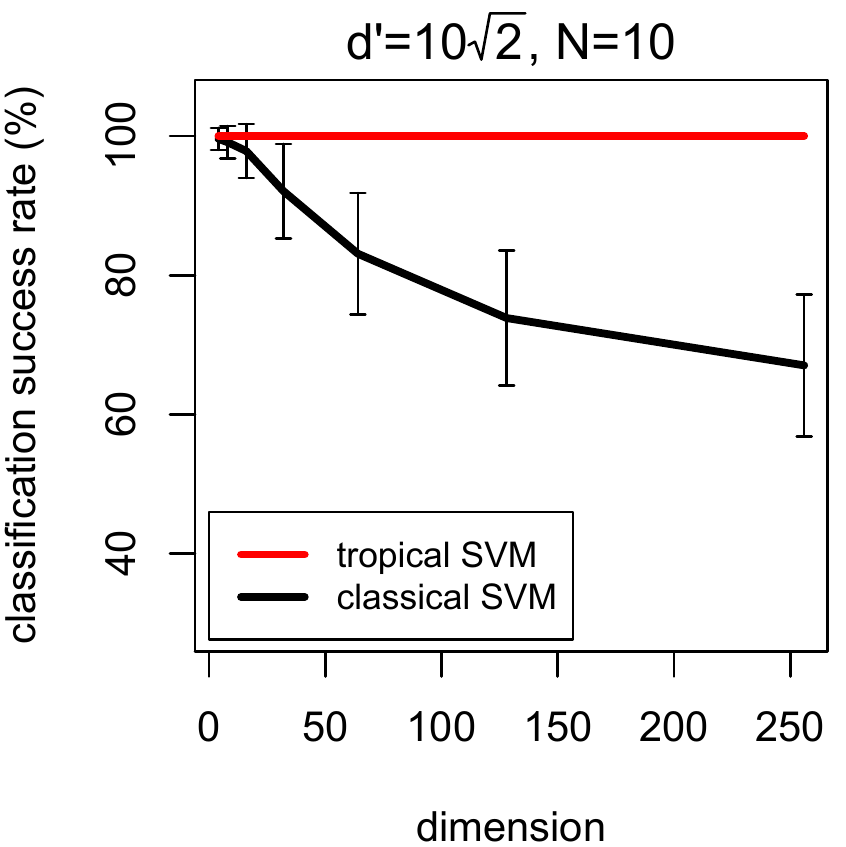} 
\caption{Classification success rates of two multivariate Gaussian distributions 
{\color{black}whose means are separated by $d'=4$ (top row) or $10\sqrt{2}$ (bottom row)
and covariances are the unit matrices.}
The numbers of training and test samples for each class are $5$ (left column) and $10$ (right column). The cross-validated performance averaged over $1000$ realizations of data were plotted.
{\color{black}The errorbars denote the standard deviations for the $1000$ realizations.}
Note that the tropical SVM is robust against the curse of dimensionality.}
\label{fig:tropSVMgaussian}
\end{figure}

As our setting of classifying two Gaussians was so simple, it is expected that the results can be quite general.
{\color{black} The only assumption here is that the difference of the two means is not parallel to ${\bf 1}=(1,1,\ldots,1)$.
We believe that in real data the difference of the means rarely becomes by chance parallel to ${\bf 1}$.}
We will examine the robustness against the curse of dimensionality for the real data 
in what follows.
But before doing that, let us explain theoretically why tropical SVMs are robust against the curse of dimensionality.

\subsection{{\color{black}Understanding} Robustness {\color{black}by Numerical Examples}}\label{sec:TheoreticalExplanation}
First, as a basic example, let us consider the simplest case with two points in three dimensions: $x_1=(5, -5, 0)$ for class $1$ and $x_2=(-5, 5, 0)$ for class $2$.
As our problem is to find $\omega=(\omega_1, \omega_2, \omega_3)$ in $\mathbb R^3 \!/\mathbb R {\bf 1}$, we can limit to  $\omega=(\omega_1, \omega_2, 0)$ without loss of generality.
By maximizing the following margin over $\omega$ in $\mathbb R^3 \!/\mathbb R {\bf 1}$ under the condition with no classification error,
\begin{equation}
M(\omega) := \min \left( d_{\rm tr}(x_1,H_\omega), d_{\rm tr}(x_2,H_\omega) \right) ,
\end{equation}
we obtain $\omega^* = (5+c, 5+c, 0)$ for $c>0$ as in Figure \ref{fig:WhyRobust}.

\begin{figure}[t!]
\centering
\includegraphics[width=0.22\textwidth]{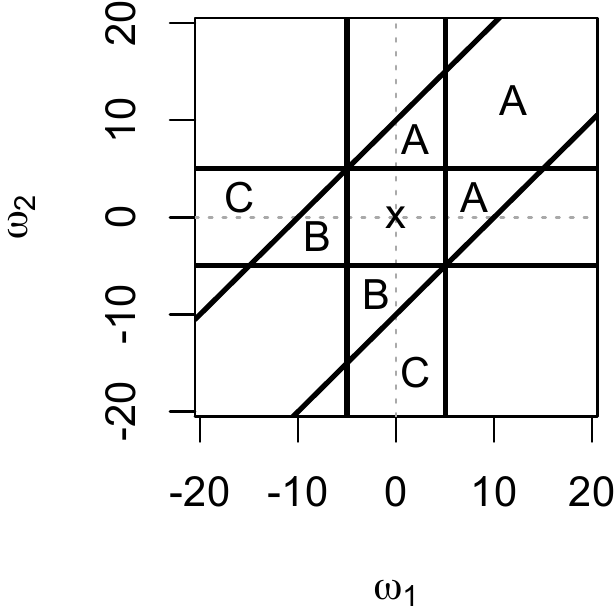} \ ~ 
\includegraphics[width=0.22\textwidth]{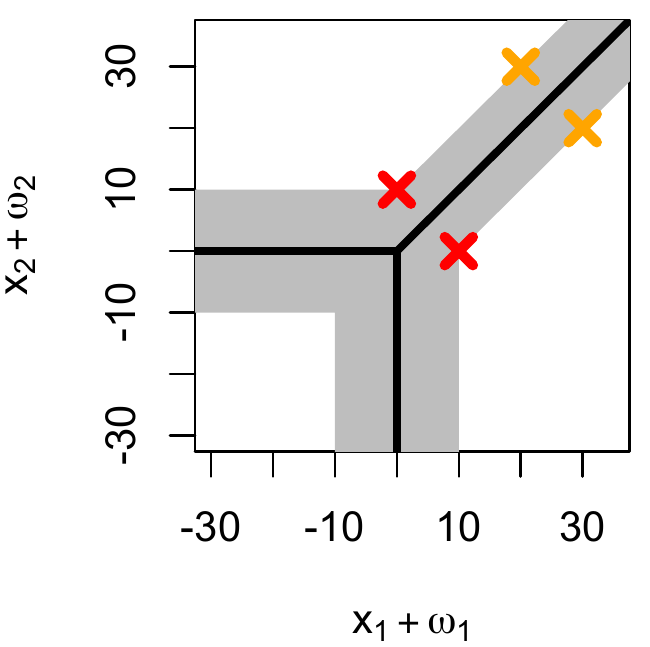}
\caption{(left) Margin $M(\omega)$ for $\omega=(\omega_1, \omega_2, 0)$.
For the regions A, it is $10-|\omega_1-\omega_2|$.
For the region $\times$ containing $(0,0)$, it is $5$+min$(\omega_1,\omega_2) = \frac{\omega_1 + \omega_2 + 10}{2} - \frac{|\omega_1-\omega_2|}{2}$.
For the regions B, it is $\frac{|\omega_1-\omega_2|}{2} - \frac{|\omega_1 + \omega_2 + 10|}{2}$.
For the regions C, it is $5-|\max(\omega_1,\omega_2)|$.
The other non-marked regions have non-zero classification errors and are ineligible.
Note that on the verge of misclassification, it is always $0$.
The maximum is attained at $\omega^* = (5+c, 5+c, 0)$ for $c>0$.
(right) Configuration in optimal solutions where margin is $10$.
The red crosses denote the data points to be classified when $\omega^* = (5, 5, 0)$.
The orange crosses denote the data points to be classified when $\omega^* = (25, 25, 0)$.
Note that both configurations attain the maximum margin.}
\label{fig:WhyRobust}
\end{figure}

In fact, $x_1 + \omega^* = (10+c,  c, 0)$ and $x_2 + \omega^* = (  c, 10+c, 0)$ lead to $d_{\rm tr}(x_1,H_{\omega^*}) = d_{\rm tr}(x_2,H_{\omega^*}) = 10$.
In Figure \ref{fig:WhyRobust} (left) for the margin $M(\omega)$, error-free regions for $\omega$ are only A, B and C.
The other regions have classification errors.
The entire shape of the error-free region reflect the Y-shape of the hyperplane $H_0$.
For example, for $\omega=(-10,-10,0)$, $x_1+\omega$ and $x_2+\omega$ belong to the same bottom left sector of $H_0$, meaning misclassification.
Note that as the margin function $M(\omega)$ is convex unless $x_1+\omega$ or $x_2+\omega$ crosses the sector borders, the linear programming is needed only once for each possible combination of labels and sectors (\cite{tang}).

Here we can impose the regularization term, $d_{\rm tr}(\omega,0)$, so that we choose the solution with the minimum norm even if the evaluation function outputs ties.
In this tradition, the unique solution is $\omega^*=( 5,  5, 0)$.

Second, let us consider the general case with $N$ points for each class in $d$ dimensions, whose coordinates are perturbed from $x_1$ or $x_2$ by standard Gaussian noises $\xi$ or $\eta$ (Figure \ref{fig:WhyRobust2}) as in the previous computational experiments.
We again need to maximize the margin:
\begin{equation}
M(\omega) := \min_{1 \leq i \leq N, ~ 1 \leq l \leq 2} d_{\rm tr}(x_{Y=l}^i,H_\omega) ,
\end{equation}
with
\begin{eqnarray}
x_{Y=1}^i = (5+\xi_1^i, -5+\xi_2^i, \xi_3^i, \xi_4^i, \ldots, \xi_d^i) , \nonumber \\
x_{Y=2}^j = (-5+\eta_1^j, 5+\eta_2^j, \eta_3^j, \eta_4^j, \ldots, \eta_d^j) ,
\end{eqnarray}
where $\xi_p^i \sim N(0,1)$ and $\eta_p^i \sim N(0,1)$ for $1 \leq p \leq d$ and $1 \leq i \leq 2$.
Let us try to find the solution near $\omega^* = (5+c, 5+c, 0, 0, \ldots, 0)$ and $d \rightarrow\infty$. 

In the previous computational experiments, we observed that for all cases of $1000$ data realizations the best classifiers used the first and the second sectors as in Figure~\ref{fig:WhyRobust2}.
Furthermore, among $x_p + \omega_p$ for $1 \leq p \leq d$, the maximum and the second maximum were attained by $x_1 + \omega_1$ and $x_2 + \omega_2$ when the training data were classified.
Therefore we focus on this situation, i.e. classification by using only sector $1$ and sector $2$, in what follows.

\begin{figure}[t!]
\centering
\includegraphics[width=0.31\textwidth]{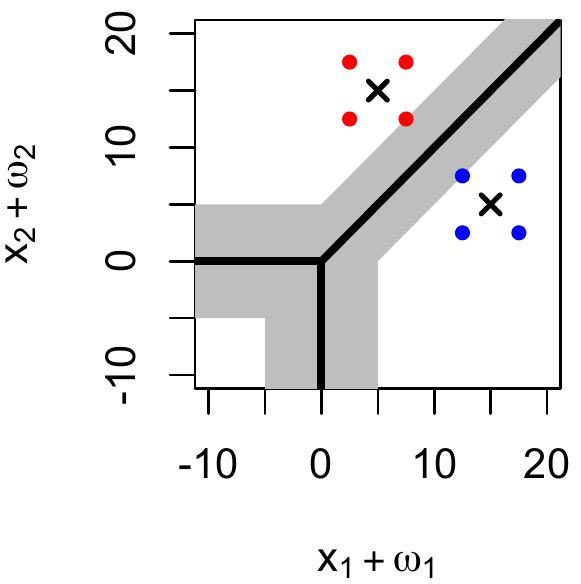}
\caption{A solution of tropical SVM that uses first and second sectors. $\omega$ is adjusted so that the tropical hyperplane is located right in the middle of the two support vectors.}
\label{fig:WhyRobust2}
\end{figure}

Under the assumption of $\omega_2 \gg 5$, which will be justified later, the tropical distance to $H_\omega$ is
\begin{eqnarray}
d_{\rm tr}(x_{Y=1}^i, H_\omega) = d_{\rm tr}(x_{Y=1}^i + \omega, H_0) \nonumber \\
= {\rm \{ max - 2nd~max~of \} }~  (5+\xi_1^i+\omega_1, -5+\xi_2^i+\omega_2, \xi_3^i+\omega_3,\nonumber\\ \ldots, \xi_{d}^i+\omega_{d}) \nonumber \\
= 10+\xi_1^i+\omega_1 - \xi_2^i - \omega_2
\end{eqnarray}
The minimum over $N$ points with class $1$ is
\begin{eqnarray}
\min_{1 \leq i \leq N} d_{\rm tr}(x_{Y=1}^i, H_\omega)
&=& 10 + \omega_1 - \omega_2 + \min_{1 \leq i \leq N} \left( \xi_1^i - \xi_2^i \right) \nonumber \\
&=& 10 + \omega_1 - \omega_2 + \xi_1^{i_*} - \xi_2^{i_*} .
\end{eqnarray}
where $i_*(\omega)$ denotes the minimizer.
Similarly for the class $2$,
\begin{eqnarray}
\min_{1 \leq j \leq N} d_{\rm tr}(x_{Y=2}^j, H_\omega)
&=& 10 + \omega_2 - \omega_1 + \min_{1 \leq j \leq N} \left( \eta_2^j - \eta_1^j \right) \nonumber \\
&=& 10 + \omega_2 - \omega_1 + \eta_2^{j_*} - \eta_1^{j_*} ,
\end{eqnarray}
where $j_*(\omega)$ denotes the minimizer.
By equating the minimum distances for the two classes,
\begin{equation}
10 + \omega_1 - \omega_2 + \xi_1^{i_*} - \xi_2^{i_*} =
10 + \omega_2 - \omega_1 + \eta_2^{j_*} - \eta_1^{j_*} ,
\end{equation}
we get
\begin{equation}
\omega_2 = \omega_1 + \frac{\xi_1^{i_*} - \xi_2^{i_*} + \eta_1^{j_*} - \eta_2^{j_*}}{2} .
\end{equation}
This theoretical prediction perfectly matched with the computational experiments.
This can be rewritten as,
\begin{eqnarray}
\omega_1 &=& 5 + c \nonumber \\
\omega_2 &=& 5 + \frac{\xi_1^{i_*} - \xi_2^{i_*} + \eta_1^{j_*} - \eta_2^{j_*}}{2} + c,
\end{eqnarray}
for $c>0$, which give the (evenly scored) maximizers of the margin under $\omega_2 \gg 5$.
Note that $\omega$ is perturbed from $(5+c,5+c)$ only slightly of order $O(1)$.
This equation illustrates that the hyperplane just listens to the positions of the support vectors.
Adding $c$ might look strange at first glance, but it is interpreted as the hyperplane's freedom to shift along the irrelevant coordinates of the support vectors.
In fact, the hyperplane for classical SVMs actually has the same freedom but you just cannot distinguish the shifted planes.

Now we can explain why tropical SVMs are robust against the curse of dimensionality.
In the computational experiments, the other $\omega_p (p\geq3)$ were smaller by around five or more.
Therefore the misclassification never happened as the maximum of $(x+\omega)$ always occurred at the first or second coordinate appropriately.

\begin{figure}[t!]
\centering
\includegraphics[width=0.34\textwidth]{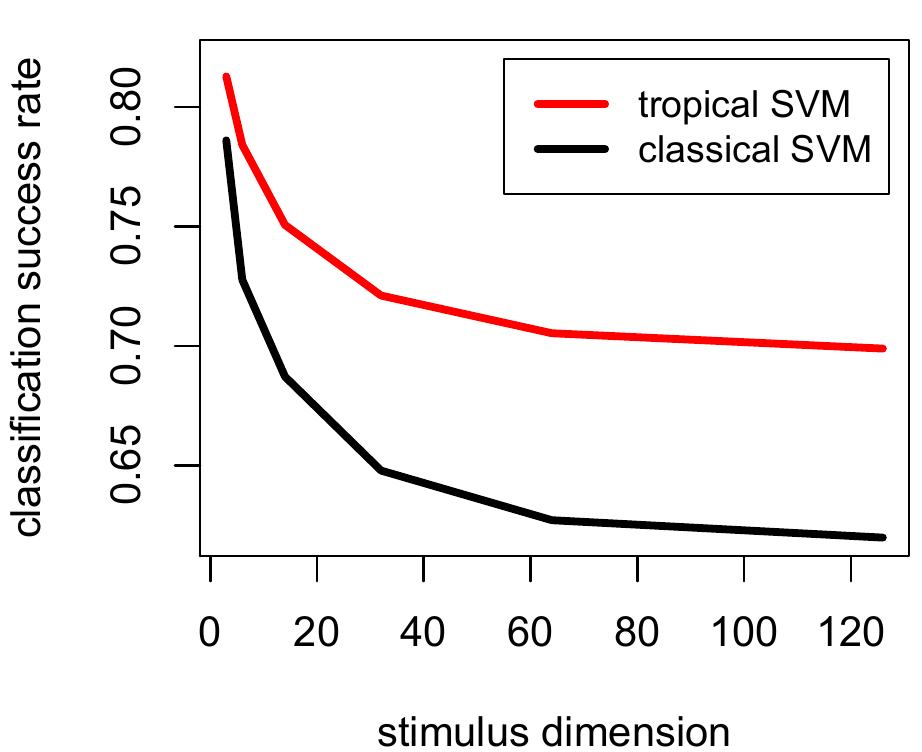}
\caption{Classification success rates of dopaminergic and nondopaminergic neurons by using various numbers of features for tropical and classical SVMs.
All features are the firing rates, i.e. the number of spikes per second, in $300$ms bins.
The first $7$ features we preferentially used are the responses to $7$ different reward stimuli, which are informative.
The following $7$ features we appended are spontaneous activities after reward stimuli, which are less informative.
The other features are spontaneous activities before reward stimuli, which are least informative.
Note that the classification success rates decreased slower for the tropical SVM than the classical SVM, suggesting that the tropical SVM is relatively robust against the curse of dimensionality.}
\label{fig:tropSVMdata}
\end{figure}

\subsection{Real Data Analysis} 
We now turn to empirical data from brain science. To summarize the data acquisition in this paragraph, we recorded the spiking activity of $493$ total neurons in the ventral tegmental area (VTA) and the substantia nigra (SN) of DAT (dopamine transporter)-cre mice using tetrodes while these mice performed various classical conditioning tasks (\cite{matsumoto16}).
We optogenetically identified $179$ dopaminergic neurons (\cite{cohen12,matsumoto16}), which enabled us to assess the classification success rates based on the neuronal tuning curves.
In each trial of the task, we delivered to a mouse one of seven stimuli, which were water rewards, air puff, etc.
The conventional tuning curve for these neurons consists of the neural responses to the seven stimuli measured as the spike count within a $300$ms time bin.

The tropical SVM for classifying dopaminergic and nondopaminergic neurons based on their activities as features was also robust against the curse of dimensionality (Figure \ref{fig:tropSVMdata}).
In order to use the same set of features, we limited to $41$ neurons including $17$ dopaminergic and $24$ nondopaminergic neurons.
That is, the sample size is $41$.
Both the training and test data consist of, randomly selected, $5$ dopaminergic and $5$ nondopaminergic neurons.
The classification performance for the test data was used to select the model or which two sectors to use in Algorithm 3 from \cite{tang}.
The cross validated classification success rates were computed for the remaining $21$ neurons.
The cross validated classification success rates averaged over the $1000$ ways of selections of training and test data were plotted in Figure \ref{fig:tropSVMdata}.

Here we did not necessarily use the entire features but rather changed the numbers of features used for the classification.
As the order of adding features can matter, we ordered and added the putatively informative features first.
After the classification success rates showed the peaks at the beginning, they decreased slower for the tropical SVM than the classical SVM (Figure \ref{fig:tropSVMdata}).

In general, we believe that although the classification success rate is enhanced by adding some most informative features at the beginning, soon it will be deteriorated by less informative features, which can be a majority in real world big data.
The tropical SVM may be superior in the latter regime of curse of dimensionality.

\section{Extension of Tropical SVMs to Function Spaces}\label{sec:functionSpace} 
{\color{black}In this section, we define tropical SVMs over a function space to enable the classification of curves, which we encounter frequently in practice.
For example, we consider neuronal tuning curves in Figure \ref{fig:distBellShapedTC} as a feature vector.
That is, we can treat a tuning curve as a point in a function space.  For example, it is valuable for identifying dopaminergic neurons in clinical researches \citep{ishikawa18,ishikawa19} to measure the distances between neuronal tuning curves and classify neurons using a tropical hyperplane boundary, which itself can be another curve in tropical SVMs.}

\subsection{Tropical Distances and Hyperplanes on Function Spaces}

Suppose we have a set of functions mapping from $\RR^d$ to $\RR$ denoted as $\mathcal{F}$ such that: 
\[
\begin{array}{c}
\mathcal{F} := \{f: \RR^s \to \RR: |f(x)| \mbox{ is bounded,} \, f(x) = f(x) + c, \\ \forall x \in \RR^s, \, \forall c \in \RR\} .
\end{array}
\]
Let $d_{\rm tr}$ be a distance between functions in  $\mathcal{F}$ such that
{\small
\[
d_{\rm tr}(f, g):= \max\{f(x) - g(x): \forall x \in \RR^s\} - \min\{f(x) - g(x): \forall x \in \RR^s\} .
\]}

\begin{example}[Functional Tropical Distances, Figure~\ref{fig:functional}]\label{ex:example:fun1}
Let $\mathcal{F}$ be a set of univariate Gaussian distribution functions with $\mu$ and $\sigma$, and mixtures of these Gaussian distributions with real valued coefficients.
Suppose $F_1$, $F_2$, $F_3$ and $F_4$ are univariate Gaussian distribution functions with $\mu_1 = \mu_2 = -2$, $\mu_3 = \mu_4 = 2$, $\sigma_1 = \sigma_3 = 1$, and $\sigma_2 = \sigma_4 = 1/2$.
Then, $d_{\rm tr}(F_1, F_2) = d_{\rm tr}(F_3, F_4) = \frac{11}{8 \sqrt{2\pi}} = 0.549$,
$d_{\rm tr}(F_1, F_3) = 0.798$, and $d_{\rm tr}(F_1, F_4) = d_{\rm tr}(F_2, F_3) =  1.197$.
\end{example}

\begin{figure*}[h!]
\centering
\includegraphics[width=0.285\textwidth]{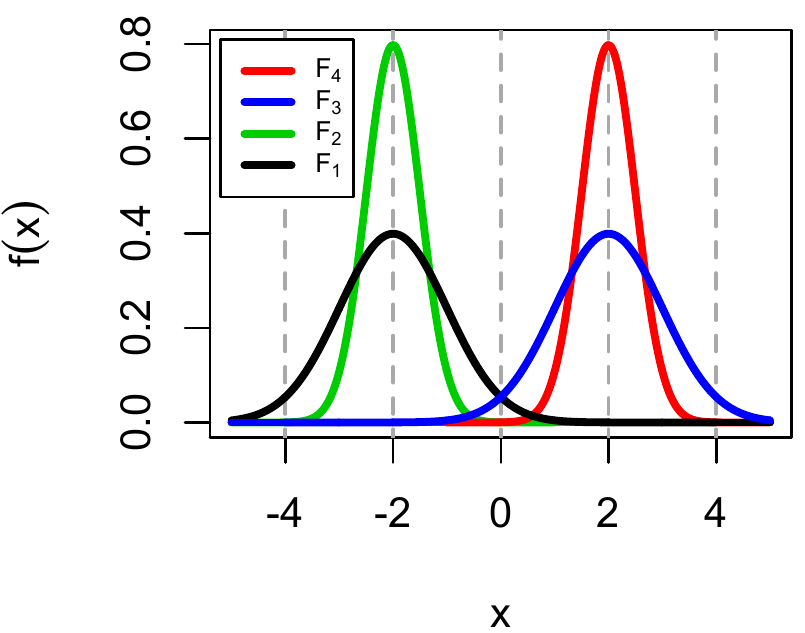} \ ~ \ ~
\includegraphics[width=0.285\textwidth]{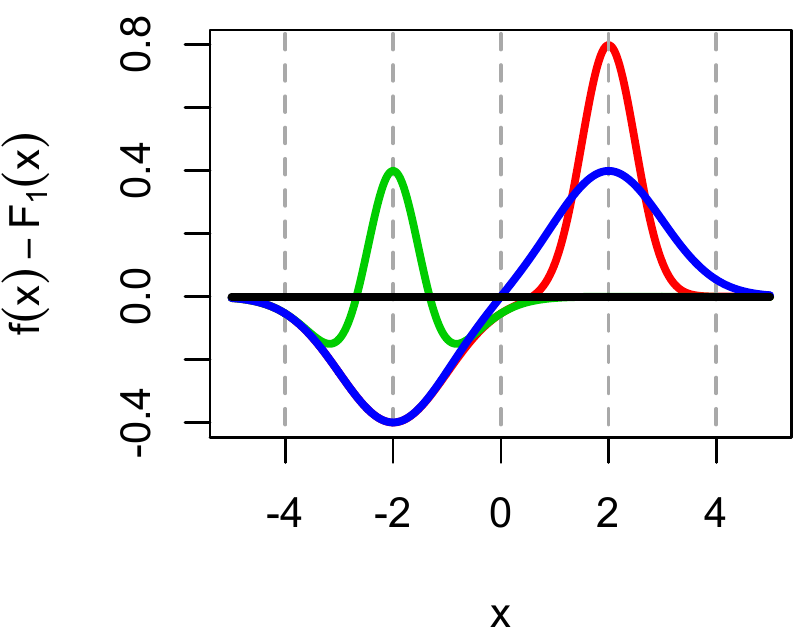} \ ~ \ ~
\includegraphics[width=0.285\textwidth]{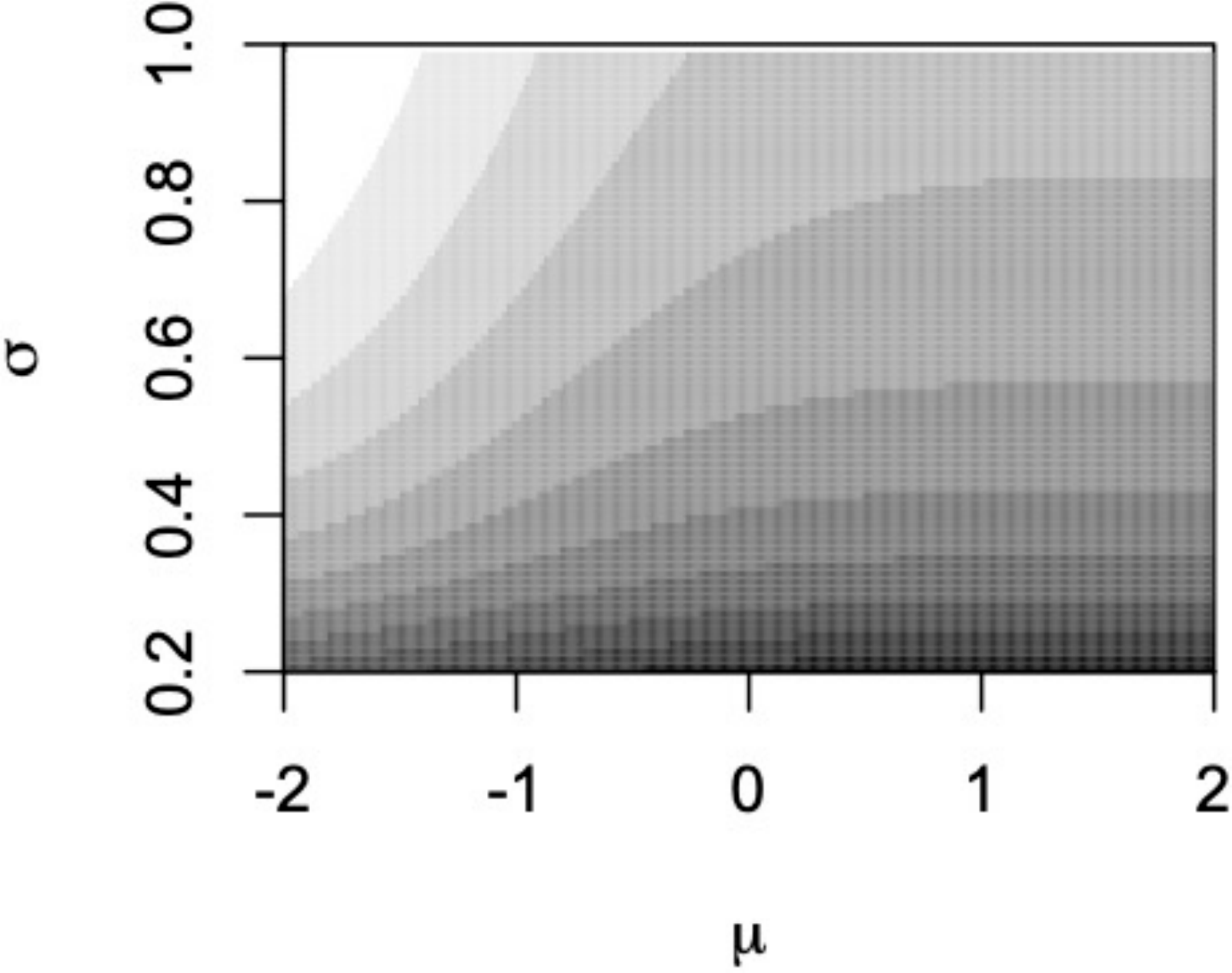}
\caption{
(left) Four example functions, $F_1$, $F_2$, $F_3$ and $F_4 \in \mathcal{F}$, which are Gaussian distribution functions with $\mu = -2$ or $2$ and $\sigma = 1$ or $0.5$.
(middle) $f(x) - F_1(x)$ for the four functions in the left figure as $f(x)$.
$F_1(x)$ denotes the reference black function with $\mu = -2$ and $\sigma = 1$.
(right) $d_{\rm tr}(f(x),F_1(x))$ is gray-color coded for various $\mu$ and $\sigma$ of $f(x)$.
Note that the distance is $0$ for $(\mu,\sigma)=(-2,1)$ and increases with increasing $\mu$ or decreasing $\sigma$: $d_{\rm tr}(F_2, F_1) = 0.549$, $d_{\rm tr}(F_3, F_1) = 0.798$, and $d_{\rm tr}(F_4, F_1) = 1.197$.
}
\label{fig:functional}
\end{figure*}

\begin{lemma}\label{lem2}
$d_{\rm tr}(f(x),g(x)) = d_{\rm tr}(-f(x),-g(x)) = d_{\rm tr}(g(x),f(x))$
\end{lemma}
\begin{proof}
\[
\begin{array}{cl}
  & d_{\rm tr}(f(x),g(x)) \\
= & -\min\{f(x)-g(x): \forall x \in \RR^s\} + \max\{f(x)-g(x): \forall x \in \RR^s\} \\ 
= & \max\{-f(x)+g(x): \forall x \in \RR^s\} - \min\{-f(x)+g(x): \forall x \in \RR^s\} \\ 
= & d_{\rm tr}(-f(x),-g(x)) \\
= & d_{\rm tr}(g(x)-f(x),0) \\
= & d_{\rm tr}(g(x),f(x)) .
\end{array}
\]
\end{proof}

\begin{lemma}
$d_{\rm tr}(f(x),0)$ is a norm on $\mathcal{F}$.
\end{lemma}
\begin{proof}
For the identity,

$d_{\rm tr}(f, 0) = 0$  if and only if $\max\{f(x): \forall x \in \RR^s\} = \min\{f(x): \forall x \in \RR^s\},$
meaning that $f(x)$ is constant. 
Next, by Lemma \ref{lem2} with $g(x)=0$
\[
d_{\rm tr}(a f, 0) = d_{\rm tr}(|a| f, 0) = |a| d_{\rm tr}(f, 0)
\]
for $a \in \RR$. For triangle inequality,
{\small
\[
\begin{array}{cl}
     & d_{\rm tr}(f(x)+g(x),0) \\
   = & \max\{f(x)+g(x): \forall x \in \RR^s\} - \min\{f(x)+g(x): \forall x \in \RR^s\} \\
   = & \max\{f(x)+g(x): \forall x \in \RR^s\} + \max\{-f(x)-g(x): \forall x \in \RR^s\} \\
\leq & \max\{f(x): \forall x \in \RR^s\}+\max\{g(x): \forall x \in \RR^s\}\\
&+\max\{-f(x): \forall x \in \RR^s\}+\max\{-g(x): \forall x \in \RR^s\} \\
   = & \max\{f(x): \forall x \in \RR^s\}+\max\{g(x): \forall x \in \RR^s\}\\
   &-\min\{f(x): \forall x \in \RR^s\}-\min\{g(x): \forall x \in \RR^s\} \\
   = & d_{\rm tr}(f(x),0) + d_{\rm tr}(g(x),0) .\\
\end{array}
\]}
\end{proof}

\begin{lemma}\label{lem3}
$d_{\rm tr}$ is a metric on $\mathcal{F}$.
\end{lemma}
\begin{proof}
$d_{\rm tr}$ is a metric because it is induced by the norm of the difference,
\[
d_{\rm tr}(f(x),g(x)) = d_{\rm tr}(f(x)-g(x),0) .
\]
\end{proof}

By the proof of Lemma \ref{lem3}, we have the following theorem:
\begin{theorem}
$(\mathcal{F}, d_{\rm tr})$ is a normed vector space.
\end{theorem}

\begin{definition}[Tropical Hyperplane in $\mathcal{F}$]
Let $B_{\epsilon}(x)$ be an open ball around $ x \in \RR^s$ with its radius $\epsilon > 0$.
For any $\omega\in \mathcal{F}$ and for $\epsilon > 0$, the {\em tropical hyperplane defined by $\omega$ and $\epsilon$}, denoted by $H_{\omega, \epsilon}$, is the set of points $f\in \mathcal{F}$ such that 
  $$|\{\omega(y) + f(y) = \omega(x^*) + f(x^*): y \in (\RR^s - B_{\epsilon}(x^*))\}| \geq 1,$$ where
  \[
x^* \in \argmax_{x \in \RR^s} \left(f(x) + \omega(x)\right).
\]
  We call $\omega \in \mathcal{F}$ the {\em normal vector} of $H_{\omega, \epsilon}$.
    \end{definition}

\begin{definition}[Tropical Distance to a Tropical Hyperplane in $\mathcal{F}$]
The {\em tropical distance} from a point $f\in \mathcal{F}$ to a tropical hyperplane $H_{\omega, \epsilon}$ is defined as
$$d_{\rm tr}(f, H_{\omega, \epsilon})\;:=\;\min\{d_{\rm tr}(f, g)\;|\;g\in H_{\omega, \epsilon}\}.$$
\end{definition}

\begin{lemma}\label{lem:hy:fun1}
$$d_{\rm tr}(f, H_{\omega, \epsilon})\;=\; d_{\rm tr}(f+\omega, H_{0, \epsilon}).$$
\end{lemma}
\begin{proof}
$H_{\omega, \epsilon}$ is defined as:
\[
\begin{array}{c}
H_{\omega, \epsilon} = \{g \in \mathcal{F}: 
|\{\omega(y) + g(y) = \omega(x^*) + g(x^*):\\ y \in (\RR^s - B_{\epsilon}(x^*))\}| \geq 1\}, 
\end{array}
\]
where
\[
x^* \in \argmax_{x \in \RR^s} (\omega(x) + g(x)).
\]
Thus
{\small
\[
\begin{array}{rl}
&d_{\rm tr}(f, H_\omega) \\
=& \Bigl\{d_{\rm tr}(f, g): |\{\omega(y) + g(y) = \omega(x^*) + g(x^*): \\
&y \in (\RR^s - B_{\epsilon}(x^*))\}| \geq 1 \Bigr\}\\
=& \Bigl\{\max\{f(x) - g(x)\} - \min\{f(x) - g(x)\}: \\
&|\{\omega(y) + g(y) 
= \omega(x^*) + g(x^*): y \in (\RR^s - B_{\epsilon}(x^*))\}| \geq 1 \Bigr\}\\
=& \Bigl\{\max\{f(x) - (h(x) - \omega(x))\} - \min\{f(x) - (h(x) - \omega(x))\}: \\
 & |\{h(y)  
= h(x^*) : y \in (\RR^s - B_{\epsilon}(x^*))\}| \geq 1 \Bigr\}\\
=& \Bigl\{\max\{(f(x) + \omega(x)) - h(x)\} - \min\{(f(x) + \omega(x)) - h(x) \}: \\
 & |\{h(y)  
= h(x^*) : y \in (\RR^s - B_{\epsilon}(x^*))\}| \geq 1 \Bigr\}\\
=& \left\{d_{\rm tr}(f+w, g): |\{g(y)  
= g(x^*) : y \in (\RR^s - B_{\epsilon}(x^*))\}| \geq 1 \right\}\\
=& d_{\rm tr}(f+\omega, H_{0, \epsilon}) .
\end{array}
\]}
\end{proof}

\begin{lemma}\label{lem:hy:fun2}
Let $x^* \in  \argmax_{x \in \RR^d}f(x)$. Then
$d_{\rm tr}(f, H_{0, \epsilon}) = \max\{f(x): x \in \RR^d\} - \max\{f(y): y \in (\RR^s - B_{\epsilon}(x^*))\}$.
\end{lemma}
\begin{proof}
Let $\alpha = \max\{f(x): x \in \RR^d\} - \max\{f(y): y \in (\RR^s - B_{\epsilon}(x^*))\}$.  Then, $g(x) = f(x) + \alpha \mathcal{I}(y)$, where $y \in (\RR^s - B_{\epsilon}(x^*))\}$ and $\mathcal{I}(y)$ is the indicator function.  Notice that $g(x)$ is in $H_{0, \epsilon}$.
\end{proof}

\begin{example}[Tropical Hyperplanes in $\mathcal{F}$]\label{ex:example:fun2}
We will use the same set up as in Example \ref{ex:example:fun1} and Figure \ref{fig:functional}. 
Let $\epsilon = 1$.  Then by Lemma \ref{lem:hy:fun2}, we have
\[
\begin{array}{rl}
&d_{\rm tr} (F_1, H_{0, \epsilon}) \\
=& \max\{F_1(x): x \in \RR\} - \max\{F_1(y): y \in (\RR - B_{\epsilon}(x^*)), \\
& x^*  = \argmax(F_1(x): x \in \RR)\}\\
=& F_1(-2) - F_1(-1)\\
= & 0.157 . 
\end{array}
\]
Similarly,
\[
\begin{array}{rl}
d_{\rm tr} (F_3, H_{0, \epsilon}) =& d_{\rm tr} (F_1, H_{0, \epsilon}) = F_1(-2) - F_1(-1) = 0.157, \\
d_{\rm tr} (F_4, H_{0, \epsilon}) =& d_{\rm tr} (F_2, H_{0, \epsilon}) = F_2(-2) - F_2(-1) = 0.690.
\end{array}
\]
Suppose we have $\omega = F_3$.  Then, by Lemma \ref{lem:hy:fun1}, we have
\[
\begin{array}{rl}
&d_{\rm tr} (F_1, H_{\omega, \epsilon}) \\
=& d_{\rm tr}(F_1 + \omega, H_{0, \epsilon})\\
=& d_{\rm tr}(F_1 + F_3, H_{0, \epsilon})\\
=& (F_1(-2) + F_3(-2)) - (F_1(2) + F_3(2))\\
= & 0.399 - 0.399\\ 
= & 0.
\end{array}
\]
Thus, $F_1 \in H_{F_3, \epsilon}$. Similarly, $F_3 \in H_{F_1, \epsilon}$, $F_2 \in H_{F_4, \epsilon}$ and $F_4 \in H_{F_2, \epsilon}$.
\end{example}

\begin{definition}[Sectors of Tropical Hyperplane in $\mathcal{F}$]
Each tropical hyperplane $H_{\omega, \epsilon}$ divides $\mathcal{F}$ into components, which are {\em open sectors},
{\small
$$
\begin{array}{l}
S_{\omega, \epsilon}^{x}:=\{\;f\in \mathcal{F}|\\ \omega(x_0)+f(x_0)>\omega(y)+f(y),\forall x_0 \in B_{\epsilon}(x),  \forall y\not \in  B_{\epsilon}(x)\}, x \in \RR^s \}.
\end{array}
$$}
\end{definition}

\begin{example}
We will use the same set up as in Example \ref{ex:example:fun1}.  Again, let $\mathcal{F}$ be a set of univariate Gaussian distribution functions with $\mu$ and $\sigma$, and mixtures of these Gaussian distributions with real valued coefficients. Suppose $\omega \equiv 0$ which is a Gaussian distribution function with $\mu = 0$ and $\sigma \to \infty$, $\epsilon = 1$, and $x = 0$. Then
$$
\begin{array}{l}
S_{\omega, 1}^{0}:=\{\mbox{Mixtures of univariate Gaussian}\\\mbox{distributions with real coefficients whose argmax in }[-1, 1]] \}.
\end{array}
$$
Also note that a set of 
\[
\begin{array}{l}
\{\mbox{Gaussian distribution functions with }\mu \in [-1, 1]\\ \mbox{ and }\sigma > L \mbox{ for some } L > 0\}
\end{array}
\]
is in this open sector.
\end{example}

\subsection{Tropical SVMs on Function Spaces} 
For a given $\epsilon > 0$, suppose $F \in \mathcal{F}$ and $Y^1, \, Y^2 \in \{0, 1\}$ are random variables such that there exist $\omega^* \in \mathcal{F}$ with
{\tiny
\begin{equation}\label{equation:tropSVMcond}
\Bigl(\bigcup_{x^* \in \argmax_{x \in \RR^s} (\omega^*(x) + F(x)| Y^1)}B_{\epsilon}(x^*)\Bigr) \bigcap  \Bigl(\bigcup_{x^* \in \argmax_{x \in \RR^s} (\omega^*(x) + F(x)| Y^2)}B_{\epsilon}(x^*) \Bigr) = \emptyset.
\end{equation}}

Now we set up a tropical SVM over $\mathcal{F}$, whose solution $\omega$ satisfies Equation~\ref{equation:tropSVMcond}.
Let $\mathcal{D}_{\mathcal{F}}$ be the distribution on the joint random variable $(F, Y)$ for $F \in \mathcal{F}$ and $Y \in \{0, 1\}$, and let $\mathcal{S}_{\mathcal{F}}$ be the sample $\mathcal{S}_{\mathcal{F}}:= \{(F^1, Y^1), \ldots , (F^n, Y^n)\}$.
For a given $\epsilon > 0$, we formulate an optimization problem for solving the normal vector $\omega \in \mathcal{F}$ of an optimal tropical separating hyperplane $H_{\omega, \epsilon}$ for random variables $X \in \mathcal{F}$ given $Y \in \{0, 1\}$: For some cost $C \in \RR$,

{\tiny
\begin{equation}\label{equation:24:fun}
\begin{matrix}
\displaystyle \max \limits_{\omega \in \mathcal{F}}\min \limits_{F(x) \in \mathcal{S}_{\mathcal{F}}} \left(
\underbrace{\left(\max_{x \in \RR^s}(F(x)+\omega(x))-\max_{z \in \mathcal{Z}}(F(z)+\omega(z))\right)}_\text{margin}+ \underbrace{\frac{C}{n} \sum_{k=1}^n \max \left\{\mathcal{I}_{B_{\epsilon}(X^*)}(x) - Y^k \right\}}_\text{error}\right), 
\end{matrix}
\end{equation}}
where $X^* \in \argmax_{x \in \RR^s} (\omega(x) + F(x))$ and $\mathcal{Z}:= \argmax_{x \in (\RR^s - B_{\epsilon}(X^*))} (\omega(x) + F(x))$.

In practice, we approximate each function $F^i$, for $i = 1, \ldots n$, by its empirical function $\hat{F}^i$ by taking some point $x \in \RR^s$ to evaluate $F^i(x)$.  In this paper we propose Algorithm \ref{alg1} to heuristically conduct a tropical SVM using empirical functions $\hat{F}^i$, for $i = 1, \ldots n$, using finite set of points $\{x^1, \ldots , x^k\} \subset \RR^s$ with $x^j \not \in B_{\epsilon}(x^i)$ for $j \not = i$ and for $i = 1, \ldots , k$.  Let $\hat{F}^i = (F^i(x^1), \ldots , F^i(x^k))$, for $i = 1, \ldots n$.

\begin{algorithm}
 \caption{Heuristic tropical SVM over $\mathcal{F}$}\label{alg1}
\begin{algorithmic}
\State{Input: A train set $\mathcal{S}=\{(\hat{F}^1, Y^1), \ldots , (\hat{F}^n, Y^n)\}$}

\State{Output: Estimated normal vector $\hat{\omega}$ of a tropical hyperplane}


 \State{Apply $\{(\hat{F}^1, Y^1), \ldots , (\hat{F}^n, Y^n)\}$ to a tropical SVM over $\RR^k / \RR {\bf 1}$}
 
\Return{The output from the tropical SVM over $\RR^k / \RR {\bf 1}$}
\end{algorithmic}
\end{algorithm}

\section{Discussion}\label{sec:dis}

We show the generalization error bounds for tropical SVMs over the tropical projection space $\RR^d / \RR {\bf 1}$ which is isometric to $\RR^{d-1}$.  These bounds still depend on the dimension $d$ and if we fix the sample size $n$, these bounds do not make sense and we cannot extend these bounds for tropical SVMs over a function space $\mathcal{F}$ with tropical metric $d_{\rm tr}$.  For future work it is interesting to obtain generalization error bounds for tropical SVMs over a function space $\mathcal{F}$ with tropical metric $d_{\rm tr}$.
In addition, computational experiments show that tropical SVMs over $\RR^d / \RR {\bf 1}$ have much lower error rates than ones of $L_2$ norm SVMs over $\RR^d$ when we fix the sample size and grow the dimension $d$.  It seems these error rates are bounded by some constant.  We are interested in tighter generalization error bounds for tropical SVMs over the tropical projection space $\RR^d / \RR {\bf 1}$. 

The generalization error bound for tropical SVMs was derived without any assumptions on data distributions.
Specifically, the bound was derived purely combinatorially only based on the shapes of the hyperplanes.
Algorithmic details of tropical SVMs are not yet taken into account.
Thus it is not surprising that the tropical SVMs under the max-plus algebra outperform in the case where a few axes are much more informative than the others.
Thus the evaluation of the tropical SVM in computational experiments gave complementary information.

The max-plus algebra also leads to the anomalous $\log n$ scaling behaviors in the tropical distance $d_{\rm tr}$ between random vectors of length $n$ and its scaling behavior is the direct consequence of the extreme value statistics.
As it is also essential for the theoretical explanation of the robustness of tropical SVMs against the curse of dimensionality, extreme value statistics play the key role  in the computation with the tropical metric.

For computational experiments, we used software developed by \cite{tang}.
\cite{tang} developed heuristic methods for hard margin and soft margin tropical SVMs over $\RR^d / \RR {\bf 1}$ since it is infeasible to compute them in practice.
We conjecture that computing optimal separating tropical hyperplanes for tropical SVMs is NP-hard. However, we still do not know the exact computational time complexity in terms of $n$ and $d$.

\printcredits

\bibliographystyle{cas-model2-names}

\bibliography{document}


\end{document}